\documentclass{article}
\usepackage{iclr2027_conference,times}

\usepackage{amsmath,amsfonts,bm}

\def\eqref#1{equation~\ref{#1}}

\def\1{\bm{1}}

\DeclareMathAlphabet{\mathsfit}{\encodingdefault}{\sfdefault}{m}{sl}
\SetMathAlphabet{\mathsfit}{bold}{\encodingdefault}{\sfdefault}{bx}{n}

\usepackage{amsmath,amssymb,amsthm,mathtools}
\usepackage{graphicx}
\usepackage{caption}
\usepackage{subcaption}
\usepackage{booktabs}
\usepackage{multirow}
\usepackage{tabularx}
\usepackage{makecell}
\usepackage{adjustbox}
\usepackage{microtype}
\usepackage{xcolor}
\usepackage{enumitem}
\usepackage{xspace}
\usepackage{fontawesome5}
\usepackage{algorithm}
\usepackage[most]{tcolorbox}
\usepackage{float}
\usepackage[section]{placeins}
\usepackage{algpseudocode}
\usepackage{hyperref}
\usepackage{url}
\usepackage{siunitx}
\usepackage{fontawesome5} 

\usepackage{colortbl}

\definecolor{decafblue}{RGB}{232,242,251}
\definecolor{decafbluehead}{RGB}{218,234,247}
\definecolor{accessgreen}{RGB}{236,247,239}
\definecolor{refgray}{RGB}{244,244,244}
\definecolor{targetorange}{RGB}{253,244,226}
\definecolor{gainpos}{RGB}{235,247,238}
\definecolor{gainneg}{RGB}{251,237,237}
\definecolor{darkblue}{rgb}{0.0, 0.0, 0.55}

\hypersetup{
    colorlinks=true,
    linkcolor=darkblue,
    citecolor=darkblue,
    urlcolor=darkblue,
}

\usepackage{appendix}
\let\appendixpagenameorig\appendixpagename
\renewcommand{\appendixpagename}{\normalfont\LARGE\scshape\appendixpagenameorig}
\usepackage{titletoc}
\usepackage{tocloft}
\usepackage{titlesec}

\definecolor{algBorder}{HTML}{D1D5DB}  
\definecolor{algBG}{HTML}{F9FAFB}      
\definecolor{algKey}{HTML}{1E40AF}      
\definecolor{algHl}{HTML}{E0F2FE}      
\definecolor{algLineNum}{HTML}{9CA3AF} 

\definecolor{decafTeal}{HTML}{1B7F80}
\definecolor{decafGold}{HTML}{B27612}
\definecolor{decafBlue}{HTML}{1769AA}
\definecolor{decafRed}{HTML}{D24A3A}
\definecolor{decafPurple}{HTML}{7B4FA3}

\algrenewcommand\algorithmicrequire{\textbf{\color{algKey}Require:}}
\algrenewcommand\algorithmicforall{\textbf{\color{algKey}for all}}
\algrenewcommand\algorithmicdo{\textbf{\color{algKey}do}}
\algrenewcommand\algorithmicend{\textbf{\color{algKey}end}}
\algrenewcommand\algorithmicreturn{\textbf{\color{algKey}return}}


\newcommand{\decaf}{\textsc{DECAF}\xspace}
\newcommand{\Abs}{\mathrm{Abs}}

\newtheorem{proposition}{Proposition}
\newtheorem{theorem}{Theorem}
\newtheorem{corollary}{Corollary}
\theoremstyle{remark}

\title{DECAF~{\large \texorpdfstring{\faCoffee}{[Coffee]}}: \\
\textbf{D}ecomposition of \textbf{E}vidence, \textbf{C}ontradiction, \textbf{a}nd \textbf{F}ragility in Perturbation Responses}

\author{Lei You \\ Technical University of Denmark \\ \textit{leiyo@dtu.dk}}

\iclrfinalcopy
\begin{document}
\maketitle

\begin{abstract}
Perturbation methods explain model decisions by measuring prediction changes
under altered inputs, but response magnitude tells us only how much a model
reacts, not what that reaction means. The same magnitude can support the final
factual--counterfactual difference, oppose it, or arise strongly along the
perturbation path yet vanish at the endpoint. We therefore track how the
contrast develops as paired inputs are progressively revealed, using the final
contrast to interpret the trajectory.
We introduce \emph{\decaf
(\textsc{D}ecomposition of \textsc{E}vidence,
\textsc{C}ontradiction, \textsc{A}nd \textsc{F}ragility)}, which routes aligned,
opposed, and endpoint-null responses into evidence $E$, contradiction $C$, and
fragility $F$. The decomposition preserves ordinary magnitude exactly,
$\Abs=E+C+F$, and is unique under endpoint-relative axioms.
Across controlled vision and tabular settings, the three components track
independently measured behavior. In a 72-model ImageNet-9 audit, we compare
cases with nearly identical response magnitude but different independently
measured behaviors. The largest \decaf component agrees with an observed
behavior in 96.4\% of cases, compared with 35.0\% for magnitude alone.
Changing only the reveal path increases total response by nearly 80\%, yet
evidence barely changes while fragility grows by more than $4\times$.
On FunnyBirds and ImageNet-1k, short forward-only \decaf trajectories outperform
the tested general-purpose attribution baselines. On a 1B-scale DINOv2 model,
a short trajectory matches a strong gradient-based baseline with $4.75\times$
lower wall time and $2.36\times$ lower peak memory. \\[2mm]
\faGithub~\textbf{Code}: {\small \url{https://github.com/youlei202/decaf}}

\end{abstract}

\section{Introduction}
\label{sec:intro}

Post-hoc explanation asks what aspects of an input matter to a model’s prediction. A natural way to answer this question is contrastively: compare the prediction with what would happen if a feature, region, concept, or context were changed. Perturbation-based and counterfactual methods instantiate this comparison by constructing alternative inputs and measuring the resulting change in model output~\citep{fong2017meaningful,dhurandhar2018explanations,goyal2019counterfactual,covert2021explaining}. This has the basic form of an effect measurement: change a factor, compare two responses, and summarize the difference. In practice, that response is often reduced to a scalar magnitude.

A magnitude answers an important question: \emph{how much did the model react?} It leaves a second question open: \emph{what did the reaction mean?} Consider an input and a paired counterfactual that changes one factor. Their prediction difference measures how strongly the model responds to that change. This endpoint comparison, however, shows only the net effect after inputs are fully observed. It does not show how the contrast emerges as information becomes available. We therefore consider a paired reveal: the paired inputs are progressively revealed along matched trajectories, so that their prediction difference can be observed at the same level of revealed information.

As the pair is revealed from an uninformative state, the contrast can evolve in different ways. It may align with the final effect, oppose it, or become large even when the fully observed pair differs little. These distinct behaviors can have identical absolute responses. Treating these responses as equivalent therefore hides behavior that matters for interpretation.

\begin{figure*}[t]
\centering
\begin{minipage}[t]{0.6\linewidth}
\begin{figure}[H]
  \vspace{0pt}
  \centering
  \includegraphics[width=\linewidth]{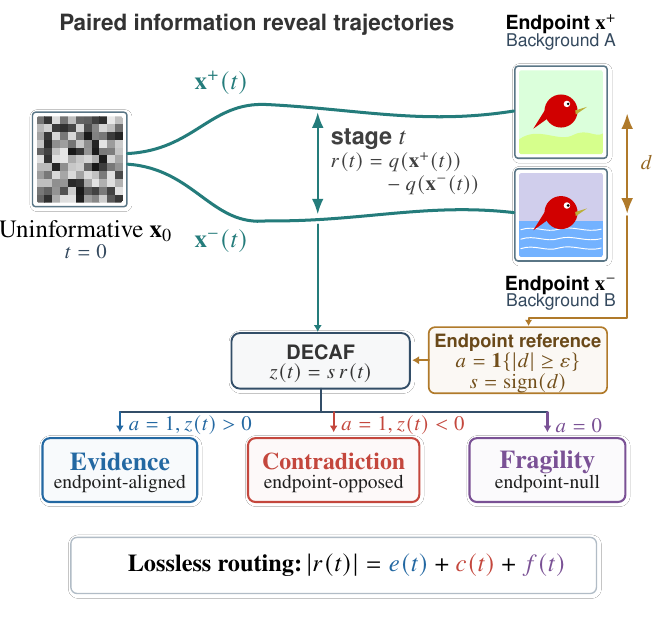}
  \caption{\textbf{\decaf at a glance.} The paired reveal measures a signed stage response, while the clean endpoint supplies an activity gate and orientation. \decaf routes the same response into endpoint-aligned evidence, endpoint-opposed contradiction, or endpoint-null fragility, with exact conservation. The compact summary reports the two headline consequences: mechanism recovery after ordinary magnitude is matched, and substantially more stable cross-protocol diagnostics.}
\label{fig:decaf-overview}
\end{figure}
\end{minipage}%
\hfill
\begin{minipage}[t]{0.38\linewidth}
We introduce \decaf in \figurename~\ref{fig:decaf-overview}, short for \emph{\textsc{D}ecomposition of \textsc{E}vidence, \textsc{C}ontradiction, \textsc{A}nd \textsc{F}ragility}. \decaf resolves this ambiguity by using the final response between the fully observed pair as a semantic reference for the entire reveal process. Its magnitude determines whether the factor has a meaningful final effect, while its sign provides the reference direction. Intermediate responses that agree with this direction are routed to \emph{evidence}, while responses that oppose it are routed to \emph{contradiction}. If the fully observed pair shows little response, there is no meaningful final effect to align with, and intermediate responses are routed to \emph{fragility}. We call this \emph{semantic routing}: the observed responses are assigned different roles according to their relation to the final contrast, without changing the trajectories themselves. The routing is lossless, so evidence, contradiction, and fragility sum exactly to the ordinary response magnitude.
  \textbf{\textbf{\decaf} is model-agnostic}, therefore explaining any black-box model that returns a scalar score.
\end{minipage}

\end{figure*}

\textbf{Relation to prior work.}
Gradient and path methods assign contributions to input coordinates or internal units~\citep{simonyan2014deep,selvaraju2017gradcam,sundararajan2017axiomatic,shrikumar2017deeplift,bach2015lrp}. Removal and sampling methods measure prediction changes under masking or replacement~\citep{fong2017meaningful,fong2019extremal,petsiuk2018rise,covert2021removing}. Counterfactual and causal approaches define meaningful contrasts~\citep{goyal2019counterfactual,chattopadhyay2019causal,janzing2020causal}. Evaluation work shows that baselines, removal operators, and off-manifold inputs can alter explanation behavior~\citep{adebayo2018sanity,hooker2019benchmark,rong2022consistent,jethani2023labelleakage}. \decaf begins after a paired response has been measured. It refines the response by its relation to the clean endpoint. Unlike positive--negative feature attribution, it also contains an endpoint-null branch. Unlike amortized explainers such as FastSHAP, it requires no learned explanation model~\citep{jethani2022fastshap}. A detailed comparison with feature attribution, counterfactual explanation, evaluation frameworks, and efficient black-box methods appears in Appendix~\ref{app:related}.

\textbf{We make five contributions}:
\begin{itemize}[leftmargin=1.4em,itemsep=1pt,topsep=0pt]
    \item \emph{Problem.}
    We identify a semantic gap in perturbation explanation: response magnitude tells us how strongly a model reacts, but not what that reaction means.

    \item \emph{Observation interface.}
    We introduce paired reveal, which tracks a factual--counterfactual contrast along matched trajectories as information becomes available.

    \item \emph{Algorithm.}
    We introduce \decaf, which uses the final contrast to route intermediate responses into evidence, contradiction, and fragility.
    The decomposition is lossless and unique under simple axioms, and adds no model queries beyond the paired trajectories.

    \item \emph{Behavioral meaning.}
    Controlled vision and tabular experiments show that the three components track learned reliance, effect reversal, and fragility across neural, linear, and tree-based models.

    \item \emph{Practical consequence.}
    On ImageNet-9, \decaf recovers response semantics after magnitude is matched and diagnoses what changes across perturbation paths.
    It also outperforms the tested general-purpose attribution baselines on FunnyBirds and ImageNet-1k and scales efficiently to a large DINOv2 vision transformer.

\end{itemize}
\section{Paired Information Revealing: Formal Setup}
\label{sec:setup}

We formalize paired information revealing. Let
$q:\mathcal{X}\rightarrow\mathbb{R}$ denote the model score.
An input $\mathbf{x}^{+}$ and a factor-level counterfactual
$\mathbf{x}^{-}$ define the final contrast
\[
d=q(\mathbf{x}^{+})-q(\mathbf{x}^{-}).
\]
The pair need not be pre-specified: in feature attribution,
$\mathbf{x}^{-}$ can be induced from $\mathbf{x}^{+}$ by a
removal or replacement operator.

A paired reveal starts from a common uninformative state $\mathbf{x}_0$
and produces matched trajectories
$(\mathbf{x}^{+}(t),\mathbf{x}^{-}(t))$ for $t\in \mathcal{T}=[0,1]$, with
\[
\mathbf{x}^{+}(0)=\mathbf{x}^{-}(0)=\mathbf{x}_0,\qquad
\mathbf{x}^{+}(1)=\mathbf{x}^{+},\qquad
\mathbf{x}^{-}(1)=\mathbf{x}^{-}.
\]

The signed response is
$r(t)=q(\mathbf{x}^{+}(t))-q(\mathbf{x}^{-}(t))$, with $d=r(1)$.
Ordinary magnitude $|r(t)|$ preserves size but discards its
relation to the final contrast; Section~\ref{sec:method} uses it
to assign semantic roles.
\section{DECAF: A Semantic Decomposition of Paired Responses}
\label{sec:method}

Section~\ref{sec:setup} gives us two quantities: the response $r(t)$ at each
reveal level and the final contrast $d=r(1)$. \decaf uses the final contrast
only as a semantic reference for interpreting the intermediate responses.
If the final contrast is substantial, its direction tells us which responses
agree with the model's eventual behavior and which oppose it. If the final
contrast is negligible, intermediate responses cannot be interpreted as
supporting or opposing a meaningful final effect.

Choose a practical threshold $\epsilon>0$ and define
\[
a=\mathbf{1}\{|d|\ge \epsilon\},
\qquad
s=\operatorname{sign}(d).
\]
The gate $a$ determines whether the final contrast is active, while $s$
orients the trajectory when it is active. Let
\[
z(t)=s\,r(t),
\]
and write $z^+=\max\{z,0\}$ and $z^-=\max\{-z,0\}$.
DECAF routes the response as
\[
(e(t),c(t),f(t))
=
\bigl(a z^+(t),\,a z^-(t),\,(1-a)|r(t)|\bigr).
\]

An active response aligned with the final contrast becomes
\emph{evidence}; an active response pointing in the opposite direction
becomes \emph{contradiction}; and a response on a final-contrast-null pair
becomes \emph{fragility}. This logic is illustrated earlier in \figurename~\ref{fig:decaf-overview}.

Let $\mu$ be a normalized measure over stages and let the outer expectation cover paired examples and protocol randomness. We report
\begin{equation}
    M=\mathbb{E}|d|,
    \qquad
    (E,C,F)=\mathbb{E}\!\int_{\mathcal{T}}(e(t),c(t),f(t))\,\mathrm{d}\mu(t),
    \qquad
    \Abs=E+C+F.
    \label{eq:population-profile}
\end{equation}
The finite-grid implementation replaces the integral by quadrature. No additional model query is needed once the signed paired trajectory is available. See Appendices~\ref{app:theory} and~\ref{app:estimators} for finite-grid estimation, conditional summaries, threshold sensitivity, and
streaming accumulation.

\textbf{Forward-only implementation.}
The model enters only through evaluations of $q$. \decaf needs no gradients, parameters, or internal activations. It can run against a neural network, a tree ensemble, or a remote score API. Examples, stages, factors, paths, and models are independent batching dimensions.

\begin{figure}[H]
  \centering
  \begin{minipage}[t]{0.43\textwidth}
    \vspace{0pt}
    \centering
    \captionof{table}{Model calls per image in ImageNet-9. }
    \label{tab:complexity}
    \scriptsize
    \setlength{\tabcolsep}{2.5pt}
    \resizebox{\linewidth}{!}{%
\begin{tabular}{lrrrr}
\toprule
Method &
\shortstack{Fwd.} &
\shortstack{Bwd.} &
\shortstack{Internal} &
\shortstack{Learned} \\
\midrule

\rowcolor{decafblue}
\decaf-9
& 18
& \cellcolor{accessgreen}\textbf{0}
& \cellcolor{accessgreen}\textbf{No}
& \cellcolor{accessgreen}\textbf{No} \\

Input$\times$Grad
& 1 & 1 & Yes & No \\

IG
& 16 & 16 & Yes & No \\

SmoothGrad
& 16 & 16 & Yes & No \\

BlurIG
& 12 & 12 & Yes & No \\

Occlusion
& 49 & 0 & No & No \\

RISE
& 256 & 0 & No & No \\

\bottomrule
\end{tabular}
}
  \end{minipage}%
  \hfill
\begin{minipage}[t]{0.55\textwidth}
  \vspace{0pt}
  \centering
  \begin{tcolorbox}[
    enhanced,
    boxrule=0.5pt,
    colframe=algBorder,
    colback=algBG,
    arc=4pt,
    left=2pt, right=2pt, top=2pt, bottom=2pt,
    shadow={0mm}{0.3mm}{0mm}{gray!15} 
  ]
    \captionof{algorithm}{\textbf{DECAF in four steps.}}
    \label{alg:decaf}
    \small
    \begin{algorithmic}[1]
      \Require score $q$, paired path $(\mathbf{x}^{+}(t),\mathbf{x}^{-}(t))$, $\varepsilon$
      \State $d\gets q(\mathbf{x}^{+})-q(\mathbf{x}^{-})$; \;$a\gets\mathbf{1}\{|d|\geq\varepsilon\}$;\;$s\gets\operatorname{sign}(d)$
      \ForAll{$t\in\mathcal{T}$}\Comment{$r(t=1)=d$ is cached and reused}
        \State $r(t)\gets q(\mathbf{x}^{+}(t))-q(\mathbf{x}^{-}(t))$
        \State \rlap{\colorbox{algHl}{\phantom{$(e,c,f)\gets\bigl(a(sr)^{+},a(sr)^{-},(1-a)|r|\bigr)$}}}$(e,c,f)\gets\bigl(a(sr)^{+},a(sr)^{-},(1-a)|r|\bigr)$
      \EndFor
      \State \Return $M,E,C,F$ by averaging over stages and examples
    \end{algorithmic}
  \end{tcolorbox}
\end{minipage}
\end{figure}


\textbf{Scope of the semantics.}
\decaf assigns observable, endpoint-relative roles to paired responses; it does not infer the latent cause of a trajectory response. These response roles are defined relative to the score $q$, the chosen factual--counterfactual pair, the reveal trajectory, and the practical threshold  $\epsilon$. In particular,
\emph{fragility} denotes response on a pair whose final contrast is negligible
under this specification. It does not by itself distinguish off-manifold
artifacts, boundary uncertainty, calibration effects, or other possible causes.
An active pair may still be path-sensitive, but its response is described
through aligned and opposed mass rather than $F$.

\section{Theoretical Foundations of Semantic Routing}
\label{sec:theory}

Section~\ref{sec:method} defines a simple routing rule from a stage response
$r(t)$ and final contrast $d$ to evidence, contradiction, and fragility.
We first ask whether this routing is arbitrary.

To formalize this routing, we establish three operational axioms that reflect practical explanation goals. First, \emph{conservation} ensures that no response magnitude is lost or invented. Second, \emph{endpoint gating} isolates sensitivity on pairs whose final contrast is negligible under the chosen threshold. Third, \emph{directional support} distinguishes intermediate responses that move toward the model's final decision from those that oppose it. 

\begin{theorem}[Unique hard-gated semantic routing]
\label{thm:canonical}
Fix $a\in\{0,1\}$, endpoint orientation $s\in\{-1,1\}$ on the active branch, and response $r\in\mathbb{R}$. The unique nonnegative triple satisfying \textbf{conservation}, $e+c+f=|r|$; \textbf{endpoint gating}, $f=0$ for $a=1$ and $e=c=0$ for $a=0$; and \textbf{directional support}, $c=0$ when $sr\geq0$ and $e=0$ when $sr\leq0$, is
\[
(e,c,f)=\bigl(a(sr)^{+},\;a(sr)^{-},\;(1-a)|r|\bigr).
\]
\end{theorem}

Theorem~\ref{thm:canonical} removes tunable mixtures once the endpoint gate and orientation are specified. It establishes uniqueness within the stated endpoint-relative semantics, not uniqueness of hard endpoint gating itself. Under these operational axioms, the routing is uniquely determined. \emph{See Appendix~\ref{app:theory} for the proof and the positive--negative (Jordan) decomposition interpretation.}

Theorem~\ref{thm:canonical} shows that the routing is unique once the endpoint-relative semantics are fixed. We next ask whether retaining the three routed components provides information beyond ordinary magnitude.
\begin{theorem}[\decaf strictly refines response magnitude]
\label{thm:nonident}
Ordinary response magnitude is a deterministic function of the \decaf
profile,
\[
\mathrm{Abs}=E+C+F,
\]
so every decision rule based on $\mathrm{Abs}$ can also be implemented from
$(E,C,F)$. The reverse recovery fails in general: for every $m>0$, the
distinct profiles $(m,0,0)$, $(0,m,0)$, and $(0,0,m)$ all produce
$\mathrm{Abs}=m$. Hence the refinement is strict for any decision problem
that distinguishes two such profiles with positive probability.
\end{theorem}
Thus \decaf is a lossless refinement of ordinary magnitude: any analysis
based on $\mathrm{Abs}$ can be reproduced from $(E,C,F)$, while the reverse
recovery is impossible in general. The refinement is strict whenever a
decision problem distinguishes response profiles that share the same
magnitude. For example, if evidence, contradiction, and fragility are equally
likely and all produce magnitude $m$, every magnitude-only classifier receives
the same observation and cannot exceed $1/3$ accuracy, whereas the full
profile distinguishes the three roles. Section~\ref{sec:imagenet9} tests this consequence empirically after explicitly
matching ordinary response magnitude.
See Appendix~\ref{app:theory} for the Bayes limit, strict information refinement, and unequal priors.

Theorem~\ref{thm:nonident} shows that magnitude can collapse distinct response roles. A practically important case is the difference between an effect that disappears and one that reverses. Both can reduce aligned evidence, but only reversal should create contradiction.

Suppose the final contrast is active with magnitude $m>0$.
 Let $\eta\in[0,1]$ denote the probability that this effect is altered by context. With probability $1-\eta$, the response remains aligned with magnitude $m$; with probability $\eta$, it is either suppressed to zero (attenuation) or reversed to $-m$ (inversion). We consider three simple response regimes: \emph{preservation}, where the effect keeps its direction; \emph{attenuation}, where it weakens toward zero; and \emph{inversion}, where it reverses direction. 

\begin{proposition}[Contradiction separates attenuation from inversion]
Under equal aligned and opposed magnitudes,
\[
{\footnotesize
\begin{array}{lccc}
\toprule
\text{Regime} & E & C & \mathrm{Abs} \\
\midrule
\text{Preservation}       & m           & 0      & m \\
\text{Attenuation to zero}& (1-\eta)m  & 0      & (1-\eta)m \\
\text{Inversion}          & (1-\eta)m  & \eta m & m \\
\bottomrule
\end{array}
}
\]
In the inversion regime, $C/(E+C)=\eta$.
\label{prop:calibration}
\end{proposition}
Attenuation and inversion can remove the same amount of aligned evidence, but only inversion creates contradiction. Thus C distinguishes loss of an effect from reversal of that effect. The controlled experiment in Section~\ref{sec:controlled}.
\emph{See Appendix~\ref{app:theory} for the proof, unequal magnitudes, and continuous mixtures.}

When contradiction and fragility vanish, conservation gives $\text{Abs}=E$. Thus DECAF reduces exactly to ordinary magnitude when the response is already fully aligned; it introduces additional distinctions only when the observed behavior contains them.

\section{Behavioral Validation Across Models and Modalities}
\label{sec:controlled}

The preceding theory defines endpoint-relative response roles. We now test whether these roles correspond to model behavior measured independently of \decaf. We begin with the 3D Shapes dataset, whose generative factors can be changed one at a time while the remaining factors are held fixed~\citep{burgess2018understanding}. This gives exact factual--counterfactual pairs and lets us formulate three controlled learning problems in which reliance, off-path sensitivity, and effect reversal can be measured independently of \decaf. Across the controlled suite, we use \emph{ResNet-18} and a small \emph{ViT}. The base grid contains \emph{30 trained models} and \emph{180 model--factor units}, with additional dedicated checkpoints and training variants for the three behavioral tests. Comparing \decaf's decomposition with these independent behavioral measurements then tests whether each component carries the intended response semantics.

\begin{figure*}[t]
    \centering
    \begin{subfigure}[t]{0.230\linewidth}
        \centering
        \includegraphics[width=\linewidth]{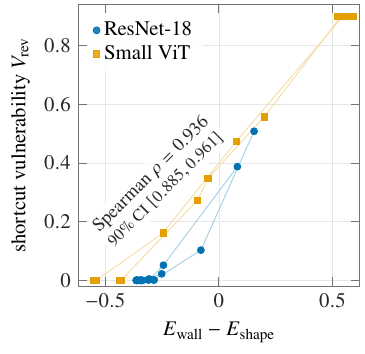}
        \caption{Population-level reliance.}
        \label{fig:e-validation-population}
    \end{subfigure}\hfill
    \begin{subfigure}[t]{0.25\linewidth}
        \centering
        \includegraphics[width=\linewidth]{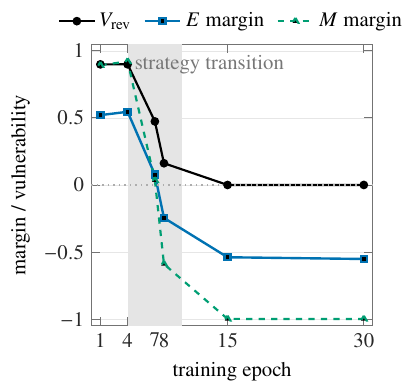}
        \caption{Within-model transition.}
        \label{fig:e-validation-trajectory}
    \end{subfigure}\hfill
    \begin{subfigure}[t]{0.234\linewidth}
        \centering
        \includegraphics[width=\linewidth]{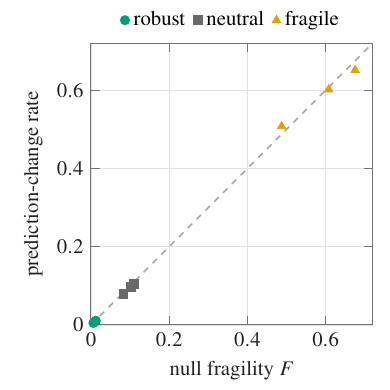}
        \caption{Endpoint-null sensitivity.}
        \label{fig:f-validation}
    \end{subfigure}\hfill
    \begin{subfigure}[t]{0.234\linewidth}
        \centering
        \includegraphics[width=\linewidth]{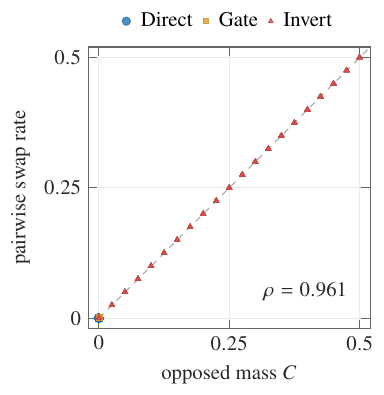}
        \caption{Label-level reversal.}
        \label{fig:c-validation}
    \end{subfigure}
    \caption{\textbf{The three response roles track independently measured model behavior.} (a--b) Evidence follows shortcut reliance across checkpoints and through a within-model strategy transition. (c) Fragility follows off-path prediction change on endpoint-null pairs. (d) Contradiction follows pairwise label reversal. Additional intervention-separation, architecture-specific, and cross-geometry diagnostics appear in Appendix~\ref{app:controlled-results}.}
    \label{fig:controlled-validation}
\end{figure*}

\textbf{Evidence.}
We train object-shape classifiers in an environment where background wall color is correlated with the label, so models may rely on either this shortcut or the object itself. Reversing the wall--label correlation at test time gives an independent measure of shortcut reliance: a wall-dependent model loses accuracy, while a shape-dependent model remains stable. Across 52 training checkpoints, the evidence margin $E_{\mathrm{wall}}-E_{\mathrm{shape}}$ correlates $0.936$ with this reversal vulnerability, with a 90\% bootstrap interval of $[0.885,0.961]$ (Figure~\ref{fig:controlled-validation}(a)). Figure~\ref{fig:controlled-validation}(b) shows the same relation during a within-model strategy transition; complete trajectories appear in Appendix~\ref{app:controlled-results}.

\textbf{Fragility.}
We construct models for which floor color has little effect once the image is fully revealed, but different effects before full reveal. To create this difference, we expose models during training to partially revealed inputs while preserving the original shape task. In \emph{fragile training}, the model is encouraged to react to changes in floor color at these partial states. In \emph{robust training}, it is instead encouraged to remain unchanged, while \emph{neutral training} uses only the original clean task without either objective. We then measure the resulting sensitivity independently using held-out floor-color interventions. If $F$ captures endpoint-null path sensitivity, it should track this prediction-change rate. It does so closely (Figure~\ref{fig:controlled-validation}(c)); intervention separation and cross-geometry checks appear in Appendix~\ref{app:controlled-results}.

\textbf{Contradiction.}
We construct three tasks with the same object-color effect in the original wall context but different behavior when the wall context changes. In \emph{Direct}, the effect is preserved; in \emph{Gate}, it disappears; and in \emph{Invert}, it reverses. The resulting label behavior---preservation, collapse, or swap---therefore provides an independent measure of what happened to the effect. If $C$ captures opposition rather than mere weakening, it should remain near zero for Direct and Gate and increase only when the effect reverses. Across 30 models and all mismatch levels, $C$ tracks the pairwise label-swap rate with $\rho=0.961$, whereas ordinary magnitude does not ($\rho=-0.036$; Figure~\ref{fig:controlled-validation}(d)). Detailed regime separation and calibration appear in Appendix~\ref{app:controlled-results}.

\textbf{Transfer across model classes and modalities.}
We next ask whether the same response roles survive beyond controlled vision models. We use a balanced 240,000-example subset of Covertype~\citep{blackard1998covertype} with 54 natural features and append a binary context and a binary candidate factor. In one set of tasks, the candidate factor has the same effect in the reference context, while changing the context either preserves, removes, or reverses that effect. In a second set, the factor remains nearly irrelevant in the reference context but can affect predictions under the alternate context. We train \emph{135 classifiers} spanning linear, tree-based, and neural models.

Crucially, these task constructions specify what the training data attempt to induce, not what the model necessarily learns. We therefore determine preservation, inversion, and endpoint-null sensitivity directly from held-out predictions. If the response roles transfer, $E$ should track realized preservation, $C$ realized inversion, and $F$ endpoint-null prediction change. See \tablename~\ref{tab:covertype-main}. 

Remark that some models trained under the inversion construction instead suppress the candidate effect, while the fragility construction becomes ordinary endpoint evidence for logistic regression. \decaf follows the behavior realized by the trained model rather than the behavior intended by the data generator; the family-level audit appears in Appendix~\ref{app:covertype-transfer}.

\begin{table*}[!t]
    \centering
    \caption{
    \textbf{Cross-model transfer of the three response roles on Covertype.}
    Entries are Spearman correlations with held-out realized behavior across
    135 classifiers. The \decaf column uses $E$ for preservation, $C$ for
    actual inversion, and $F$ for alternate-context prediction change
    restricted to endpoint-null pairs. Brackets give 95\% joint family/seed
    bootstrap intervals. $\dagger$ SHAP interaction is available only for tree
    models. Complete baselines, model-family results, threshold-conditioned
    analyses, and measured cost appear in
    Appendix~\ref{app:covertype-transfer}.
    }
    \label{tab:covertype-main}

    \small
    \setlength{\tabcolsep}{5.0pt}

    \begin{tabular}{lllrrrr}
        \toprule
        Held-out behavior
        & Role
        & \cellcolor{decafbluehead}\decaf $\rho$
        & $\Abs$
        & $M$
        & Native SHAP
        & SHAP inter. \\
        \midrule

        Preservation
        & $E$
        & \cellcolor{decafblue}\textbf{0.864}\,[0.832,0.895]
        & 0.804
        & 0.657
        & 0.781
        & $-0.251^{\dagger}$ \\

        Actual inversion
        & $C$
        & \cellcolor{decafblue}\textbf{0.987}\,[0.973,0.997]
        & $-0.207$
        & $-0.001$
        & $-0.205$
        & $0.093^{\dagger}$ \\

        Endpoint-null change
        & $F$
        & \cellcolor{decafblue}\textbf{0.974}\,[0.942,0.988]
        & 0.588
        & 0.148
        & 0.481
        & $0.069^{\dagger}$ \\

        \bottomrule
    \end{tabular}
\end{table*}

\section{Real-World Audit: Response Semantics on ImageNet-9}
\label{sec:imagenet9}

Section~\ref{sec:controlled} established the meanings of $E$, $C$, and $F$ in controlled learning problems. We now ask whether the same response semantics describe behavior on natural images, including models that were never trained for our audit. ImageNet-9 is useful for this purpose because its background variants keep the foreground object fixed while changing only the background~\citep{xiao2021noise}. This gives natural factual--counterfactual pairs in which we can ask whether a model's response to background change behaves as evidence, contradiction, or endpoint-null sensitivity.

\textbf{Setup.} We evaluate two complementary groups of models on these same background interventions. The first consists of 24 off-the-shelf ImageNet-1k classifiers. These models were trained independently of our experiment, so they test whether the response semantics established in Section~\ref{sec:controlled} remain meaningful for existing natural-image models. Their 1,000-way predictions are aggregated into the nine ImageNet-9 superclasses before evaluation.

The second group contains 48 models trained directly on ImageNet-9. We fine-tune six backbones---\emph{ResNet-50, ConvNeXt-Tiny, EfficientNet-B3, RegNetY-8GF, Swin-T, and ViT-B/16}---on four versions of the training data: the original images, images with randomly reassigned backgrounds, images with a fixed next-class background, and foreground-only images. These training conditions deliberately produce models with different degrees and directions of background dependence. This group therefore tests whether the same response semantics continue to track behavior as the model's learned use of background changes. Together, the two groups form a 72-model zoo: one provides models trained independently of our audit, while the other provides controlled diversity in the behavior being audited.

\begin{figure*}[t]
\centering
\begin{subfigure}[t]{0.325\linewidth}
    \centering
    \includegraphics[width=\linewidth]{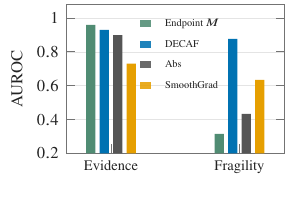}
    \caption{Behavioral alignment.}
\end{subfigure}\hfill
\begin{subfigure}[t]{0.34\linewidth}
    \centering
    \includegraphics[width=\linewidth]{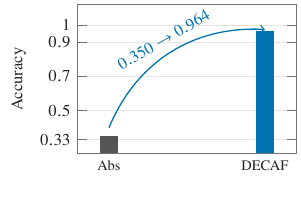}
    \caption{Response-role accuracy.}
\end{subfigure}\hfill
\begin{subfigure}[t]{0.33\linewidth}
    \centering
    \includegraphics[width=\linewidth]{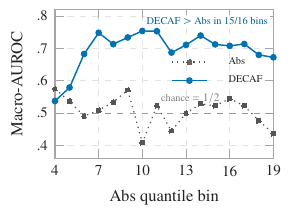}
    \caption{Across magnitude bins.}
\end{subfigure}

\caption{\textbf{Response semantics transfer to ImageNet-9.}
(a) Evidence and fragility align with independently measured natural-image behaviors.
(b) Nearly identical response magnitudes can correspond to different response semantics, which DECAF preserves.
(c) The advantage persists across magnitude bins.}
\label{fig:imagenet9-mechanism}
\end{figure*}

\begin{figure*}[t]
\centering
\begin{subfigure}[t]{0.46\linewidth}
    \centering
    \includegraphics[width=\linewidth]{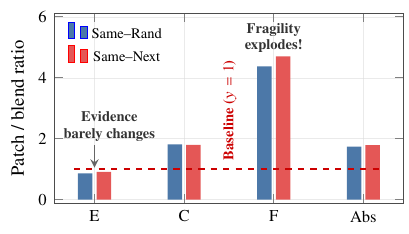}
    \caption{Response composition under patch reveal.}
\end{subfigure}
\hspace{5mm}
\begin{subfigure}[t]{0.39\linewidth}
    \centering
    \includegraphics[width=\linewidth]{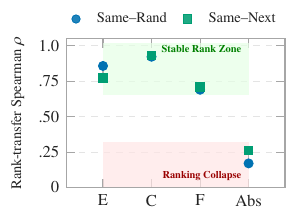}
    \caption{Model-rank transfer across reveal paths.}
\end{subfigure}

\caption{\textbf{Reveal paths change response composition and ranking.}
(a) Patch reveal increases ordinary magnitude mainly through contradiction and fragility rather than evidence.
(b) Ordinary-magnitude rankings transfer poorly across reveal paths, while evidence and contradiction remain substantially more stable.}
\label{fig:protocol-audit}
\end{figure*}

\subsection{Response semantics survive natural images}

The first question is whether the response roles validated in
Section~\ref{sec:controlled} remain connected to independently measured
behavior on natural images. We consider two behaviors. Background reliance
asks whether changing only the background substantially disrupts the model's
prediction. Endpoint-null sensitivity asks whether a pair with little final
background effect can nevertheless remain sensitive to held-out background
corruptions. Both behaviors are defined independently of \decaf. If the
response semantics transfer, $E$ should identify the former and $F$ the latter.

For background reliance, we mark a pair as positive when the model correctly classifies the image with a same-class background, but replacing that background with a random-class one either changes the predicted class or lowers the true-class probability by at least $0.20$. We then ask whether these behavior-positive pairs tend to receive larger response scores. Using $E$ directly as a ranking score gives an AUROC of $0.930$: a value of $0.5$ corresponds to random ranking and $1$ to perfect separation. Ordinary response magnitude reaches $0.900$, while endpoint magnitude reaches $0.960$ (Figure~\ref{fig:imagenet9-mechanism}(a)). The strong endpoint result is expected because this behavioral target is itself defined by a prediction change at the endpoint.

For endpoint-null sensitivity, we first restrict attention to pairs whose final background effect is negligible, $|d|<0.02$. We then apply a separate set of background corruptions that are not used to construct the reveal path. A pair is marked positive if any of these held-out corruptions changes the predicted class or shifts the true-class probability by at least $0.20$. We ask whether these independently identified sensitive pairs tend to receive larger $F$ values than the remaining endpoint-null pairs. $F$ reaches AUROC $0.878$, compared with $0.433$ for ordinary response magnitude, $0.316$ for endpoint magnitude, and $0.635$ for SmoothGrad (Figure~\ref{fig:imagenet9-mechanism}(a)).

Together, the two tests show complementary behavior: $E$ tracks consequential background use when the final effect is present, whereas $F$ exposes sensitivity that remains when that final effect is negligible. Appendix~\ref{app:imagenet9-setup} gives the full pair construction, held-out corruption set, behavioral-label definitions, and baseline details.

\subsection{Same Response Magnitude, Different Response Semantics}

We next ask whether the decomposition reveals information that ordinary
response magnitude does not already contain. We reuse the independently defined
evidence and fragility indicators from Section~6.1 and add a contradiction
indicator: contradiction behavior is present when changing only the background
makes the model switch from the foreground class to the class associated with
the new background. These three behavioral indicators may overlap.

We then compare cases with different behavioral patterns but nearly identical
ordinary response magnitudes. Requiring their $\Abs$ values to differ by at
most $5\%$ yields 8,289 matched comparisons. Because a case may exhibit more
than one behavior, we do not force it into a single ground-truth class.
Instead, we ask whether the largest of $E$, $C$, and $F$ corresponds to a
behavior that is actually present; overlap and ties are handled as detailed in
Appendix~\ref{app:imagenet9-setup}.

Under this test, a single magnitude provides no basis for preferring evidence,
contradiction, or fragility. $\Abs$ reaches a role-agreement accuracy of
$0.350$, whereas \decaf reaches $0.964$
(Figure~\ref{fig:imagenet9-mechanism}(b)). Thus, nearly identical response
magnitudes can accompany substantially different observed behaviors, while
their decomposition preserves this distinction.

The conclusion also holds beyond the matched cases. Within narrow bins of
ordinary response magnitude, macro-AUROC rises from $0.517$ for $\Abs$ to
$0.677$ for \decaf, and \decaf exceeds $\Abs$ in 15 of the 16 bins that
contain all three behavioral indicators
(Figure~\ref{fig:imagenet9-mechanism}(c)). Appendix~\ref{app:imagenet9-setup}
gives the exact matching rule, overlap and tie handling, per-bin support, and
complete baseline comparisons.

\subsection{What Changes When the Reveal Path Changes?}

We finally ask whether the same endpoint information can produce different
conclusions when it is revealed differently. We keep the model and both
factual--counterfactual endpoints fixed. In one path, the whole image gradually
changes from a common blurred image toward each endpoint; in the other, the
same endpoint information is revealed region by region. We call them the
blend and patch paths.

Ordinary response magnitude increases by about $1.8\times$ under the patch
path, but the increase is not uniform across response roles
(Figure~\ref{fig:protocol-audit}(a)). Evidence remains close to its blend value,
contradiction grows by about $1.8\times$, and fragility by more than
$4\times$. This is consistent with their semantics: with the endpoints fixed,
evidence remains oriented by the same final contrast, whereas fragility
explicitly measures sensitivity along the chosen reveal path. The pattern is
not guaranteed by the definition, but here the larger ordinary response comes
primarily from fragility rather than additional evidence.

We also ask whether the reveal path changes conclusions about how models
compare. For each model, we average $E$, $C$, $F$, and $\Abs$ over test images
and rank the models separately by each quantity under the blend and patch
paths. Spearman correlation measures the agreement between the two rankings.
Ordinary-magnitude rankings are highly unstable
($\rho=0.17/0.26$ for Same--Rand/Same--Next), whereas evidence rankings remain
at $0.86/0.77$ and contradiction rankings near $0.93$. Fragility is
intermediate at $0.69/0.71$, consistent with its greater path dependence
(Figure~\ref{fig:protocol-audit}(b)). Patch-order and threshold checks are
reported in Appendix~\ref{app:imagenet9-results}.

\section{External Attribution Benchmarks and Large-Model Scaling}
\label{sec:forward-attribution}

Sections~\ref{sec:controlled}--\ref{sec:imagenet9} validate the response
semantics of \decaf. We now ask whether it also works as ordinary feature
attribution under targets defined without $E$, $C$, or $F$. Here the pair is
constructed from the feature intervention: $\mathbf{x}^{+}$ is the original
image and $\mathbf{x}^{-k}$ changes only part or patch $k$ using a fixed
removal or replacement rule; $E_k$ ranks the features. FunnyBirds evaluates
semantic parts under two separate replacement interventions, while ImageNet-1k
IDSDS evaluates 16 fixed patches by single-patch deletion. Quality is the
per-image Spearman correlation between the attribution ranking and this
intervention-based ranking. Both datasets use
\emph{ResNet-50, VGG-16, and ViT-B/16}, with every method evaluated on the same
eligible images within each model. \decaf-3/5/9 use three, five, or nine reveal
stages.

\begin{table*}[!t]
\centering

\begin{minipage}[t]{0.555\textwidth}
\centering

\caption{
External attribution quality and measured ImageNet-1k compute.
Quality is macro-averaged equally over the three architectures; time and memory
are ImageNet-1k measurements.
KernelSHAP is shown separately because it directly queries the same patch-deletion
intervention used to define the ImageNet-1k evaluation target, giving it
evaluation-specific information unavailable to the general baselines and \decaf.
}
\label{tab:forward-attribution-main}

\setlength{\tabcolsep}{2.2pt}

\resizebox{\textwidth}{!}{
\begin{tabular}{llrrrr}
\toprule
Method
&
Access
&
\shortstack[r]{FunnyBirds\\$\rho\uparrow$}
&
\shortstack[r]{ImageNet-1k\\$\rho\uparrow$}
&
\shortstack[r]{ms/img\\$\downarrow$}
&
\shortstack[r]{Peak GiB\\$\downarrow$}
\\
\midrule

\multicolumn{6}{l}{\emph{DECAF}} \\

\rowcolor{decafblue}
\decaf-3
& Forward
& 0.372
& 0.359
& 22.1
& 11.2
\\

\rowcolor{decafblue}
\decaf-5
& Forward
& 0.403
& 0.367
& 36.7
& 11.2
\\

\rowcolor{decafblue}
\decaf-9
& Forward
& \textbf{0.406}
& \textbf{0.379}
& 65.9
& 11.2
\\

\addlinespace

\multicolumn{6}{l}{\emph{General-purpose attribution}} \\

DeepLIFT
& Backward
& 0.197
& 0.341
& 3.1
& 6.6
\\

IG-32
& Backward
& 0.271
& 0.242
& 28.7
& 50.4
\\

IG-U-32
& Backward
& 0.200
& 0.295
& 28.7
& 50.4
\\

RISE-512
& Sampling
& 0.302
& 0.179
& 216.6
& 14.4
\\

\addlinespace

\multicolumn{6}{l}{\emph{Endpoint-only reference}} \\

\rowcolor{refgray}
Endpoint $M$
& Endpoint
& 0.324
& 0.371
& --
& --
\\

\addlinespace

\multicolumn{6}{l}{\emph{Benchmark-aligned reference (uses the evaluation intervention)}} \\

\rowcolor{targetorange}
KernelSHAP-512
& Sampling
& 0.299
& 0.447
& 216.4
& 14.4
\\

\bottomrule
\end{tabular}
}

\end{minipage}%
\hspace{0.030\textwidth}%
\begin{minipage}[t]{0.395\textwidth}
\centering

\caption{DINOv2 ViT-g/14 on 238 PartImageNet strict-common-support images.}
\label{tab:dinov2g-stress}

\setlength{\tabcolsep}{3.0pt}

\resizebox{\textwidth}{!}{
\begin{tabular}{lrrr}
\toprule
Method
&
\shortstack[r]{Spearman\\$\uparrow$}
&
\shortstack[r]{sec/img\\$\downarrow$}
&
\shortstack[r]{Peak GiB\\$\downarrow$}
\\
\midrule

\rowcolor{decafblue}
\decaf-3
& 0.208
& 0.190
& 5.53
\\

\rowcolor{decafblue}
\textbf{\decaf-5}
& 0.215
& 0.300
& 5.66
\\

\rowcolor{decafblue}
\decaf-9
& 0.220
& 0.526
& 6.33
\\

IG-16
& \textbf{0.222}
& 0.710
& 9.01
\\

IG-32
& 0.213
& 1.425
& 13.39
\\

GradientSHAP
& 0.171
& 0.407
& 27.67
\\

SmoothGrad-16
& 0.042
& 0.438
& 26.43
\\

DeepLIFT
& 0.068
& \textbf{0.084}
& 7.19
\\

\bottomrule
\end{tabular}
}

\vspace{0.45em}


\captionof{table}{
Gain by the trajectory beyond endpoint-only $M$.
Entries are paired
$\Delta\rho=\rho(\mathrm{DECAF})-\rho(M)$
with 95\% confidence intervals.
Larger values indicate better.
}
\label{tab:endpoint-trajectory-main}

\centering
\setlength{\tabcolsep}{3.5pt}

\resizebox{\textwidth}{!}{
\begin{tabular}{lcc}
\toprule
Contrast
&
\shortstack{FunnyBirds\\different intervention}
&
\shortstack{ImageNet-1k\\same deletion}
\\
\midrule

\decaf-3 $- M$
&
\cellcolor{gainpos}
+.049 [.033, .064]
&
\cellcolor{gainneg}
$-$.012 [$-$.014, $-$.010]
\\

\decaf-5 $- M$
&
\cellcolor{gainpos}
+.080 [.064, .096]
&
\cellcolor{gainneg}
$-$.004 [$-$.007, $-$.002]
\\

\decaf-9 $- M$
&
\cellcolor{gainpos}
+.083 [.067, .099]
&
\cellcolor{gainpos}
+.007 [.004, .010]
\\

\bottomrule
\end{tabular}
}

\end{minipage}

\end{table*}

All three \decaf trajectories exceed every listed general-purpose baseline on
both datasets (Table~\ref{tab:forward-attribution-main}). KernelSHAP reaches
$0.447$ on ImageNet-1k, but it directly queries the same deletion game used to
define that target and is therefore shown separately. The full 50,000-image
ImageNet check preserves the ordering
$\decaf\text{-}5>\mathrm{IG\text{-}U\text{-}32}>\mathrm{IG\text{-}32}$.
Appendix~\ref{app:forward-attribution} also reruns the official FunnyBirds
RISE-6000 Single Deletion protocol and reproduces the published score
within $0.005$ on all three architectures~\citep{hesse2023funnybirds}

The compute advantage becomes clearer at larger model scale. On the \emph{1B-parameter} model
\emph{DINOv2 ViT-g/14}, \decaf-5 and IG-32 have nearly identical quality
($0.215$ versus $0.213$), while IG-32 requires $4.75\times$ the wall time and
$2.36\times$ the peak allocated memory
(Table~\ref{tab:dinov2g-stress}).

Table~\ref{tab:endpoint-trajectory-main} clarifies what the trajectory adds.
On FunnyBirds, evaluation changes parts differently from the operation used
to form $\mathbf{x}^{-k}$; the trajectory adds about $0.08$ Spearman beyond
$M$. ImageNet-1k instead evaluates the same patch deletion used to form
$\mathbf{x}^{-k}$, so $M$ is already aligned with the target; five stages
nearly match it, while nine stages add only $0.007$. Endpoint information therefore
carries most of the attribution signal when evaluation repeats the same
intervention. The trajectory adds its clearest value when the ranking must
transfer to a different intervention. Complete intervals, per-architecture
results, and boundary cases are in Appendix~\ref{app:forward-attribution}.

\section{Discussion and Conclusion}
\label{sec:discussion}

Perturbation magnitude tells us how strongly a model responds, but responses
of the same size can reflect different behaviors. \decaf uses the final
factual--counterfactual contrast to separate intermediate responses into
evidence, contradiction, and fragility, while preserving ordinary magnitude
exactly, $\Abs=E+C+F$.

The experiments show that these distinctions are not artifacts of the
definitions. They track independently measured behavior in controlled vision
and tabular settings, remain distinguishable on natural images after magnitude
is controlled, and remain useful under external attribution criteria.
ImageNet further shows why this matters: changing only the reveal path can
greatly increase total response without increasing evidence.

These response roles are relative to the chosen counterfactual pair, score,
and reveal path; they are not claims about hidden causal mechanisms. The
endpoint most directly describes its own intervention, while the trajectory is
most useful beyond that intervention. Magnitude tells us how much a model
reacts; \decaf preserves that quantity while revealing what kind of response
produced it.

\newpage
\appendix
\appendixpage

{
\small
\setcounter{tocdepth}{1}
\startcontents[sections]
\printcontents[sections]{l}{1}{}
}

\section{Detailed Related Work}
\label{app:related}

The closest literatures all study model response, but they attach different objects and semantics to that response. We organize the comparison around the object being explained and the question it answers.

\subsection{Gradient, path, and signed attribution}

Gradient explanations measure local sensitivity, while path methods accumulate sensitivity from a baseline to the input~\citep{simonyan2014deep,selvaraju2017gradcam,sundararajan2017axiomatic,smilkov2017smoothgrad,xu2020attribution,kapishnikov2021guided,you2026joint}. Layer-wise relevance propagation, DeepLIFT, FullGrad, and SHAP instead decompose a prediction into additive feature contributions under conservation, reference, or game-theoretic principles~\citep{bach2015lrp,shrikumar2017deeplift,srinivas2019fullgrad,lundberg2017unified}. Some of these methods preserve positive and negative contributions~\citep{dhurandhar2018pertinent}.

DECAF studies a different object. It does not distribute one prediction across input coordinates. It decomposes the scalar response of an already specified paired intervention. The clean endpoint supplies the orientation, and the endpoint-null branch separates fragility from positive or negative evidence. Spatial saliency and DECAF are therefore complementary: saliency indicates where a response is localized, whereas DECAF indicates what semantic role the paired response plays.

\subsection{Perturbation, counterfactual, and causal explanation}

Occlusion, Meaningful Perturbations, Extremal Perturbations, RISE, LIME, and removal-based Shapley methods construct interventions and summarize the resulting model changes~\citep{zeiler2014visualizing,fong2017meaningful,fong2019extremal,petsiuk2018rise,ribeiro2016why,covert2021removing}. Counterfactual and contrastive methods retrieve or synthesize changes that alter a decision or support a contrast~\citep{goyal2019counterfactual,chang2019counterfactual,dhurandhar2018pertinent,pmlr-v258-you25a,gu2026discover}. Causal attribution and concept-intervention methods interpret relevance through explicit interventions on inputs, representations, or human-interpretable concepts~\citep{chattopadhyay2019causal,janzing2020causal,goyal2020cace,kim2018tcav,koh2020conceptbottleneck}.

DECAF takes the paired intervention as input rather than searching for it. Counterfactual construction defines the factor and identification assumptions; the paired trajectory measures the response; endpoint orientation then decomposes that response into evidence, contradiction, and fragility. The same decomposition can therefore sit on top of different removal operators or counterfactual generators.

\subsection{Baselines, off-manifold effects, and faithfulness evaluation}

Baseline and path choices can substantially change attribution~\citep{sturmfels2020baselines,kindermans2019unreliability}, and perturbations can move samples away from the data manifold~\citep{frye2020manifold}. Sanity checks, retraining-based evaluations, infidelity measures, and debiased removal metrics test whether explanations reflect the model rather than artifacts of the evaluation protocol~\citep{adebayo2018sanity,hooker2019benchmark,yeh2019infidelity,rong2022consistent}. Other work documents label leakage, explanation fragility, and adversarial manipulation~\citep{jethani2023labelleakage,ghorbani2019interpretation,dombrowski2019explanations,slack2020fooling}.

DECAF does not claim path invariance. It keeps the intervention protocol explicit and asks which component changes when the protocol changes. This distinction matters in ImageNet-9: nested-patch reveal raises total response by roughly $80\%$, but the increase comes primarily from contradiction and fragility rather than evidence. DECAF addresses a semantic ambiguity that remains even when a paired response has been measured faithfully.

\subsection{Efficiency, black-box access, and benchmark context}

Black-box perturbation methods trade model access for repeated queries; RISE and dense occlusion can require many forward evaluations, while amortized methods such as FastSHAP add a separate explainer-training problem~\citep{petsiuk2018rise,jethani2022fastshap}. DECAF instead reuses the signed scores already computed by a paired perturbation curve. The routing step adds no model queries, requires no gradients or internal activations, and can be batched across examples, stages, factors, and models.

ImageNet-9 is a standard setting for studying background reliance and distribution shift~\citep{xiao2021noise,geirhos2020shortcut}. DECAF asks a finer question than whether a model responds to background: is that response endpoint-aligned evidence, endpoint-opposed contradiction, or sensitivity that appears only when the endpoint effect is null?

\FloatBarrier
\section{Theoretical Foundations and Proofs}
\label{app:theory}

This appendix develops the compact theory used in the main text. We begin with the pointwise decomposition, then lift it to a trajectory, a population, and a decision problem. The proofs require only integrability of the signed response and do not assume differentiability, smoothness, or a particular model class.

\subsection{Positive and negative parts}
\label{app:jordan}

For a real number $u$, define $u^{+}=\max\{u,0\}$ and $u^{-}=\max\{-u,0\}$. Then
\begin{equation}
    u=u^{+}-u^{-},
    \qquad
    |u|=u^{+}+u^{-},
    \qquad
    u^{+}u^{-}=0.
    \label{eq:positive-negative}
\end{equation}
For an integrable function $z:\mathcal T\rightarrow\mathbb R$, the signed measure $\nu(A)=\int_A z(t)\,\mathrm d\mu(t)$ admits the Jordan decomposition $\nu=\nu^{+}-\nu^{-}$, where $\mathrm d\nu^{+}=z^{+}\mathrm d\mu$ and $\mathrm d\nu^{-}=z^{-}\mathrm d\mu$. DECAF applies this decomposition to the endpoint-oriented response on the active branch and reserves a separate branch for endpoint-null pairs.

\subsection{Canonicality}

\begin{proof}[Proof of Theorem~\ref{thm:canonical}]
Fix $a$, $s$, and $r$. If $a=0$, endpoint gating requires $e=c=0$. Conservation then forces $f=|r|$. Hence the triple is unique.

Now suppose $a=1$. Endpoint gating gives $f=0$. Let $z=sr$. If $z\geq0$, directional support requires $c=0$ and conservation gives $e=|r|=z=z^{+}$. If $z\leq0$, directional support requires $e=0$ and conservation gives $c=|r|=-z=z^{-}$. Thus
\begin{equation}
    (e,c,f)=\bigl(a(sr)^{+},a(sr)^{-},(1-a)|r|\bigr).
\end{equation}
The construction is nonnegative, conserves $|r|$, obeys the gate, and has the stated support. Therefore it is the unique triple satisfying the axioms.
\end{proof}

The theorem also yields a projection interpretation. On an active pair, $e$ is the magnitude of the Euclidean projection of $r$ onto the endpoint-aligned ray $\{s\lambda:\lambda\geq0\}$, while $c$ is the magnitude of the projection onto the opposite ray. The null branch is not a projection. It records that the endpoint does not provide a stable orientation at the chosen threshold.

\subsection{Magnitude non-identifiability and strict refinement}

\begin{proof}[Proof of Theorem~\ref{thm:nonident}]
Fix $m>0$. Consider three responses. First, let the endpoint be active with $s=1$ and let $r=m$. Then $(e,c,f)=(m,0,0)$. Second, keep the same endpoint and let $r=-m$. Then $(e,c,f)=(0,m,0)$. Third, let the endpoint be null and let $r=m$. Then $(e,c,f)=(0,0,m)$. In every case, the observed magnitude is $|r|=m$. A statistic that observes only $|r|$ therefore cannot distinguish the three mechanisms.
\end{proof}

\begin{corollary}[Strict response refinement]
\label{cor:strict-refinement}
The DECAF profile is at least as informative as ordinary magnitude for every decision problem, because $\Abs=E+C+F$ is a deterministic function of the profile. It is strictly more informative for any mechanism-identification problem that assigns positive probability to two distinct profiles with the same magnitude.
\end{corollary}

\begin{proof}
Any rule that uses $\Abs$ can be composed with the map $(E,C,F)\mapsto E+C+F$. Theorem~\ref{thm:nonident} gives distinct profiles that share one magnitude. A decision problem that rewards correct mechanism identification separates those profiles, so no magnitude-only rule can match the best rule that observes the full profile. This is the standard logic of strict Blackwell refinement~\citep{blackwell1953equivalent}.
\end{proof}

A useful finite case makes the gap concrete. Suppose evidence, contradiction, and fragility are equally likely and all produce magnitude $m$. Every magnitude-only classifier receives the same observation, so its best accuracy is $1/3$. Under unequal priors, the best magnitude-only accuracy is the largest prior. The DECAF profile identifies the mechanism exactly in this idealized construction.

\subsection{Preservation, attenuation, and inversion}

\begin{proof}[Proof of Proposition~\ref{prop:calibration}]
Let the endpoint orientation be $s$. Under preservation, the stage response is $sm$, so $E=m$, $C=0$, and $\Abs=m$.

Under attenuation, the response is $sm$ with probability $1-\eta$ and zero with probability $\eta$. Taking expectations gives $E=(1-\eta)m$, $C=0$, and $\Abs=(1-\eta)m$.

Under inversion, the response is $sm$ with probability $1-\eta$ and $-sm$ with probability $\eta$. The first branch contributes $m$ to evidence, while the second contributes $m$ to contradiction. Therefore $E=(1-\eta)m$, $C=\eta m$, and $\Abs=m$. The ratio is $C/(E+C)=\eta$.
\end{proof}

The equal-magnitude assumption isolates the inversion probability. If the aligned and opposed branches have magnitudes $m_{+}$ and $m_{-}$, then
\begin{equation}
    E=(1-\eta)m_{+},
    \qquad
    C=\eta m_{-},
    \qquad
    \frac{C}{E+C}=\frac{\eta m_{-}}{(1-\eta)m_{+}+\eta m_{-}}.
    \label{eq:unequal-inversion}
\end{equation}
Thus the opposed fraction measures probability mass only when the two branches have comparable effect size. In general, $C$ measures opposed causal mass.

\subsection{Population factorization}
\label{app:population-factorization}

Let $\pi=\Pr(|d|\geq\varepsilon)$. Define branch-conditional summaries whenever the corresponding branch has positive probability. Since evidence and contradiction vanish on the null branch, while fragility vanishes on the active branch,
\begin{equation}
    E=\pi E^{\mathrm{active}},
    \qquad
    C=\pi C^{\mathrm{active}},
    \qquad
    F=(1-\pi)F^{\mathrm{null}}.
    \label{eq:population-factorization}
\end{equation}
Unconditional quantities combine branch prevalence and conditional intensity. Conditional quantities answer a different question and become unstable when their branch is rare. The factorization is therefore useful when an external outcome is defined only within the active or null branch, but the primary experiments report unconditional population summaries.

\subsection{Invariances}
\label{app:invariances}

\paragraph{Endpoint swap.}
Swapping $\mathbf x^{+}$ and $\mathbf x^{-}$ multiplies both $d$ and $r(t)$ by $-1$. The oriented response $s r(t)$ is unchanged. Hence $M,E,C,F,$ and $\Abs$ are invariant.

\paragraph{Score translation.}
Replacing $q$ by $q+b$ leaves every difference unchanged.

\paragraph{Positive affine score reparameterization.}
Let $q'=\lambda q+b$ with $\lambda>0$. Then
$d'=\lambda d$, $r'(t)=\lambda r(t)$, and $s'=s$.
If the endpoint threshold is transformed consistently as
$\varepsilon'=\lambda\varepsilon$, the gate is unchanged and
\[
(M',E',C',F',\Abs')
=
\lambda(M,E,C,F,\Abs).
\]
Hence all normalized component proportions are invariant.

With a numerically fixed threshold $\varepsilon$, however, the gate obeys
\[
a'_{\varepsilon}
=
\mathbf{1}\{|d|\ge\varepsilon/\lambda\}.
\]
Consequently,
\[
E'_{\varepsilon}=\lambda E_{\varepsilon/\lambda},\qquad
C'_{\varepsilon}=\lambda C_{\varepsilon/\lambda},\qquad
F'_{\varepsilon}=\lambda F_{\varepsilon/\lambda}.
\]
Thus positive scaling preserves component proportions only when gate
membership is preserved. Otherwise, the discrepancy is controlled by
the threshold-mass bound in Appendix~B.8.

\paragraph{Negative score scaling.}
A negative scaling reverses the semantic meaning of the target score. The endpoint orientation still makes the component magnitudes invariant after multiplying by $|\lambda|$, but the analyst has changed the behavior being explained. We therefore define the score direction before analysis.

\subsection{Protocol reparameterization}
\label{app:reparameterization}

Let $h:\widetilde{\mathcal T}\rightarrow\mathcal T$ be an increasing bijection and define the reparameterized path $\widetilde r(u)=r(h(u))$. Pointwise components are reindexed:
\begin{equation}
    \widetilde e(u)=e(h(u)),
    \quad
    \widetilde c(u)=c(h(u)),
    \quad
    \widetilde f(u)=f(h(u)).
\end{equation}
If the integration measure is pushed forward consistently, $\widetilde\mu=h^{-1}_{\#}\mu$, then the integrated profile is invariant. If both paths are instead integrated with a uniform coordinate measure, their AUCs may differ. Dynamic DECAF summaries are therefore protocol-relative, and every experiment reports the path and stage measure.

\subsection{Threshold stability}
\label{app:threshold-theory}

Let $0\leq\varepsilon_1<\varepsilon_2$ and suppose the per-example path summary is bounded by $B$. The two gates differ only on pairs satisfying $\varepsilon_1\leq|d|<\varepsilon_2$. Consequently, for each component $G\in\{E,C,F\}$,
\begin{equation}
    |G_{\varepsilon_2}-G_{\varepsilon_1}|
    \leq
    B\,\Pr\!\left(\varepsilon_1\leq|d|<\varepsilon_2\right).
    \label{eq:threshold-bound}
\end{equation}
The threshold is stable when little endpoint mass lies near the boundary. We report sensitivity over multiple thresholds rather than treating one numerical cutoff as universal.

\FloatBarrier
\section{Estimators, Complexity, and Implementation}
\label{app:estimators}

DECAF reuses the signed scores of a paired perturbation curve. Once those scores are available, the decomposition is elementwise.

\subsection{Finite-grid estimators}

Let $t_1,\ldots,t_T$ be stage points with nonnegative quadrature weights $w_1,\ldots,w_T$ that sum to one. For pair $i$, define
\begin{equation}
    E_i=\sum_{j=1}^{T}w_j e_i(t_j),
    \quad
    C_i=\sum_{j=1}^{T}w_j c_i(t_j),
    \quad
    F_i=\sum_{j=1}^{T}w_j f_i(t_j).
\end{equation}
For $N$ pairs, the empirical estimates are
\begin{equation}
    \widehat E=\frac1N\sum_{i=1}^{N}E_i,
    \quad
    \widehat C=\frac1N\sum_{i=1}^{N}C_i,
    \quad
    \widehat F=\frac1N\sum_{i=1}^{N}F_i,
    \quad
    \widehat M=\frac1N\sum_{i=1}^{N}|d_i|.
    \label{eq:finite-estimator}
\end{equation}
Every implementation checks the exact sample-level identity $|r_i(t_j)|=e_i(t_j)+c_i(t_j)+f_i(t_j)$. We use trapezoidal weights on ordered grids unless stated otherwise. When the protocol contains randomness, factual and counterfactual branches share that randomness; repeated paths are averaged or retained as clusters for uncertainty estimation.

\subsection{Query complexity}
\label{app:complexity}

Suppose we evaluate $K$ factors, $T$ stages, $J$ counterfactual maps per factor, and $R$ protocol repetitions. If factual branch evaluations are shared across factors, the number of model evaluations per example is
\begin{equation}
    RT(1+KJ).
    \label{eq:query-complexity}
\end{equation}
For one factor and one map this is $2RT$. DECAF adds no model evaluations to the ordinary signed paired trajectory; its arithmetic cost is linear in the number of scalar responses.

The ImageNet-9 main path uses $T=9$, so it requires 18 forward evaluations per pair. Under our configurations, Integrated Gradients and SmoothGrad each use 16 forward--backward evaluations, BlurIG uses 12, Occlusion uses 49 forward evaluations, and RISE uses 256. These counts describe access and query structure rather than hardware-independent wall time.

\subsection{Memory, batching, and black-box models}

A streaming implementation stores the endpoint sign, current stage scores, and three accumulators. Its additional memory is $O(BK)$ scalars for batch size $B$ and $K$ factors, beyond the model and input batch. Examples, branches, stages, factors, counterfactual maps, repetitions, and models provide independent batching dimensions.

DECAF requires no differentiability. It applies to neural networks, tree ensembles, simulators, and remote APIs whenever a stable real-valued score is available. For stochastic models, branches should share random seeds or use repeated queries. Hard labels are formally sufficient, but probabilities, logits, margins, regression values, or action values produce more informative endpoint effects.

\FloatBarrier
\section{Controlled Experimental Setup}
\label{app:controlled-setup}

The controlled suite gives DECAF an identifiable target before we move to natural images. Every experiment uses exact factor interventions in 3D Shapes, shared protocol randomness across branches, and two architectures. The experiments differ in the behavior they are designed to activate.

\subsection{Dataset and factors}

3D Shapes contains $480{,}000$ rendered scenes generated from six factors: floor color, wall color, object color, object size, object shape, and object orientation~\citep{burgess2018understanding}. The Cartesian product is fully enumerated. This structure gives us exact counterfactuals: to intervene on one factor, we change its index and hold the other five indices fixed.

We use binary prediction tasks derived from the factors. The base benchmark contains five tasks:
\begin{itemize}[leftmargin=1.4em,itemsep=2pt,topsep=3pt]
    \item \textbf{Object color:} a binary split of the object-color values.
    \item \textbf{Wall color:} a binary split of the wall-color values.
    \item \textbf{Object shape:} the two shapes that are common to all model-response supports.
    \item \textbf{Color--shape XOR:} the exclusive-or of binary object color and shape.
    \item \textbf{Context gate:} a color decision whose active factor depends on wall context.
\end{itemize}
For each task we train ResNet-18 and a small ViT with three random seeds. The resulting 30 models are crossed with all six factors, producing 180 model--factor units.

\subsection{Clean endpoint pairs}

For a factual scene $\mathbf{x}^{+}$ and target factor $k$, the counterfactual $\mathbf{x}^{-}$ changes only factor $k$. Binary factors are flipped. Multivalued color, size, and orientation factors use fixed involutive maps so that applying the map twice returns the original value. We use two independent maps for multivalued factors and average their measurements.

The clean endpoint effect is computed on 8,192 held-out factual--counterfactual pairs. The main threshold classifies a pair as active when the absolute target-score difference exceeds the experiment-specific $\varepsilon$. The endpoint classification is a property of the model and pair, not of the task label.

\subsection{Primary reveal protocol}
\label{app:cmmr}

The main controlled protocol is covariance-matched Gaussian reveal. Let $\boldsymbol\mu$ and $\widehat{\boldsymbol\Sigma}$ be the empirical mean and covariance of the image distribution. For $\alpha\in[0,1]$,
\begin{equation}
\mathbf{x}_{\alpha}
=
\boldsymbol\mu
+
\sqrt{\alpha}(\mathbf{x}-\boldsymbol\mu)
+
\sqrt{1-\alpha}\,\boldsymbol\eta,
\qquad
\boldsymbol\eta\sim\mathcal N(\mathbf{0},\widehat{\boldsymbol\Sigma}).
\label{eq:cmmr-path}
\end{equation}
The factual and counterfactual branch share the same $\boldsymbol\eta$. At $\alpha=1$ the path reaches the clean input; at $\alpha=0$ the state is independent of the individual sample under the fitted Gaussian reference. The covariance match makes the stimulus second-order neutral in whitened data coordinates. DECAF does not depend on this particular path; it only consumes the paired scores that the path produces.

The base benchmark uses 4,096 dynamic factual pairs, three noise seeds, two counterfactual maps, and 21 points along a trace-matched covariance family. The primary CMMR endpoint is accompanied by a trace-matched pixel-Gaussian endpoint. A diagonal covariance and a trace-normalized covariance-power path provide held-out geometries.

\subsection{Alternative protocols}

The covariance interpolation is
\[
\boldsymbol\Gamma_{\lambda}
=(1-\lambda)\widehat{\boldsymbol\Sigma}
+
\lambda\tau\mathbf I,
\qquad
\tau=\frac{\operatorname{tr}(\widehat{\boldsymbol\Sigma})}{D},
\]
where $D$ is the pixel dimension. The two endpoints use data covariance and trace-matched isotropic pixel covariance. The diagonal held-out path uses $\operatorname{diag}(\widehat{\boldsymbol\Sigma})$. The power path raises covariance eigenvalues to a power and renormalizes the trace. These paths test whether a semantic conclusion depends on one second-order geometry.

We use the protocol family as an audit, not as a learned worst-case score. The main paper reports CMMR-based DECAF and uses the alternatives to characterize transfer.

\subsection{Evidence validation through training trajectories}
\label{app:e-setup}

The evidence experiment trains shape classifiers in an environment where wall color is strongly correlated with the label. We use two training correlations, $0.95$ and $0.99$, two architectures, two seeds, and save checkpoints throughout training. This creates eight trajectories and 52 selected checkpoints.

The external target is shortcut-reversal vulnerability,
\[
V_{\mathrm{rev}}
=
\operatorname{Acc}(p_{\mathrm{test}}=0.95)
-
\operatorname{Acc}(p_{\mathrm{test}}=0.05).
\]
A shortcut-dominant checkpoint performs well when wall color remains aligned and fails when the correlation reverses. A shape-dominant checkpoint is stable. We compare $V_{\mathrm{rev}}$ with the evidence margin $E_{\mathrm{wall}}-E_{\mathrm{shape}}$.

\subsection{Fragility validation}
\label{app:f-setup}

The fragility experiment uses a binary object-shape task and treats floor color as the candidate endpoint-null factor. We train robust, neutral, and fragile variants for each architecture and three seeds, producing 18 models.

The neutral variant uses clean training only. The robust variant is encouraged to keep its output stable under floor-color counterfactuals at intermediate states. The fragile variant is encouraged to respond to floor color away from the endpoint while preserving the clean shape task. The intended comparison is restricted to models for which clean accuracy remains high and floor color remains endpoint-null.

The behavioral target is the prediction-change rate under held-out floor-color interventions at intermediate states. We also measure randomized-floor stability and transfer across covariance geometries.

\subsection{Contradiction validation}
\label{app:c-setup}

The contradiction benchmark uses two binary variables: object color $A$ and wall context $G$. The factual endpoint is evaluated in the context $G=1$, where all three tasks share the clean rule $Y=A$. Under the swapped context $G=0$:
\[
\text{Direct: }Y=A,
\qquad
\text{Gate: }Y=H,
\qquad
\text{Invert: }Y=1-A,
\]
where $H$ is object shape. Direct preserves the object-color effect. Gate removes it. Invert reverses it.

We train 30 models: three tasks, two architectures, and five seeds. Balanced validation accuracy is at least $0.9987$. A binary symmetric context channel swaps the wall context with probability $\eta\in[0,0.5]$. All intermediate images remain in the support of the generated dataset. Two independent wall-color involutions test map transfer.

The label-level outcomes are preserve, collapse, and swap. The key contrast is between Gate and Invert: both reduce aligned evidence, but only Invert should create endpoint-opposed mass.

\subsection{Statistical summaries}

The base benchmark uses stratified bootstrap over task, architecture, and seed. The evidence experiment bootstraps complete trajectories. The fragility and contradiction experiments report seed-level variation and architecture-stratified summaries. Bootstrap intervals use 500 repetitions unless stated otherwise. We treat the model or training trajectory as the main statistical unit rather than counting every image as independent evidence.

\FloatBarrier
\section{Selected Controlled Results}
\label{app:controlled-results}

This appendix expands the behavioral validation in Section~\ref{sec:controlled} without repeating its main tests. We report five positive extensions: a matched-magnitude example, the complete response atlas, all evidence-training trajectories, fragility checks across interventions and reveal geometries, and contradiction checks across regimes, architectures, seeds, and counterfactual maps.

\subsection{Matched-magnitude response roles}
\label{app:controlled-overview}

Figure~\ref{fig:controlled-overview} gives a concrete example of the scalar ambiguity discussed in the main text. In an object-shape model, the true shape factor and an endpoint-null floor-color factor produce nearly equal ordinary response magnitude, but their response roles differ sharply. The same distinction matters across the controlled benchmark: ordinary magnitude can rank an endpoint-null factor above a supported factor, whereas endpoint-oriented routing removes these false-null reversals.

\begin{figure*}[t]
    \centering
    \begin{subfigure}[t]{0.46\textwidth}
        \centering
        \includegraphics[
            height=3.75cm,
            width=\linewidth,
            keepaspectratio
        ]{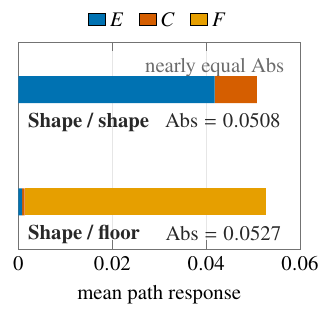}
        \caption{Nearly equal magnitude, different response roles.}
    \end{subfigure}
    \hspace{0.04\textwidth}
    \begin{subfigure}[t]{0.30\textwidth}
        \centering
        \includegraphics[
            height=3.35cm,
            width=\linewidth,
            keepaspectratio
        ]{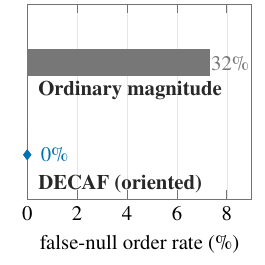}
        \caption{Orientation removes false-null reversals.}
    \end{subfigure}
    \caption{\textbf{Equal response magnitude can hide different behavior.}
    (a) The true shape factor and an endpoint-null floor-color factor have nearly identical ordinary magnitude, yet \decaf separates evidence from fragility.
    (b) Across the benchmark, ordinary magnitude ranks an endpoint-null factor above a supported factor in $7.32\%$ of comparable pairs, whereas endpoint-oriented routing reduces this rate to zero.}
    \label{fig:controlled-overview}
\end{figure*}

\subsection{Complete decomposition atlas}
\label{app:controlled-atlas}

Figure~\ref{fig:app-controlled-atlas} displays every task--factor combination. Endpoint-null factors concentrate near the fragility corner, direct color tasks concentrate in evidence, and interaction tasks carry more contradiction and mixed mass. Across 125 endpoint-null units, 96.8\% of ordinary response is fragility on average. Across 48 endpoint-supported units, the mean composition is 64.1\% evidence, 12.5\% contradiction, and 23.4\% null fragility.

\begin{figure*}[t]
    \centering
    \includegraphics[width=0.7\textwidth]{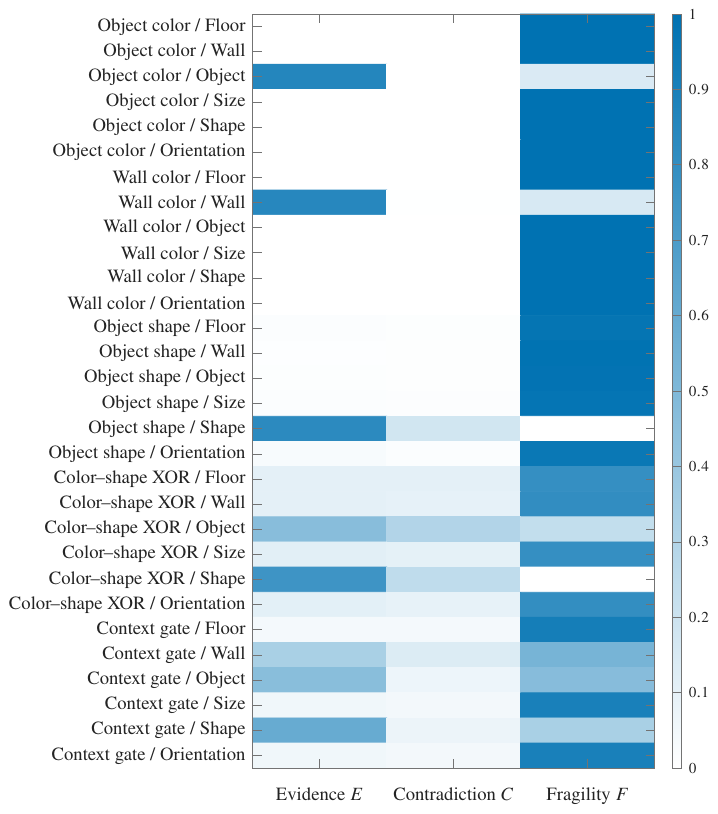}
    \caption{\textbf{Complete controlled decomposition atlas.}
    Results cover five tasks, two architectures, three seeds, and six factors. Endpoint-null factors are dominated by fragility, while task-relevant and interaction factors contain evidence and contradiction in different proportions.}
    \label{fig:app-controlled-atlas}
\end{figure*}

\subsection{Evidence across training trajectories}
\label{app:e-trajectories}

Figure~\ref{fig:app-e-heatmaps} places exact shortcut vulnerability and the \decaf evidence margin on the same eight checkpoint trajectories. Six trajectories exhibit nonconstant vulnerability, and all six show positive within-trajectory association between evidence margin and shortcut vulnerability. The two Small-ViT trajectories trained at the strongest shortcut correlation remain shortcut-dominant, and \decaf correspondingly keeps their evidence margin high.

\begin{figure*}[t]
    \centering
    \begin{subfigure}[t]{0.47\textwidth}
        \centering
        \includegraphics[
            height=3.05cm,
            width=\linewidth,
            keepaspectratio
        ]{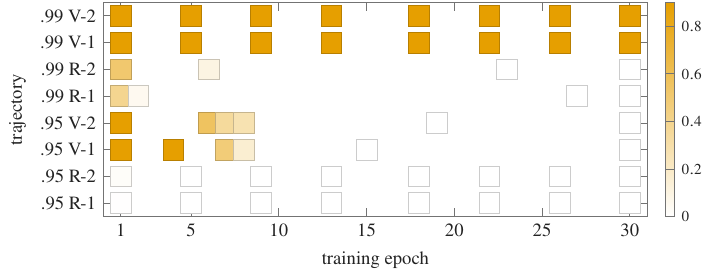}
        \caption{Exact shortcut vulnerability.}
    \end{subfigure}
    \hfill
    \begin{subfigure}[t]{0.47\textwidth}
        \centering
        \includegraphics[
            height=3.05cm,
            width=\linewidth,
            keepaspectratio
        ]{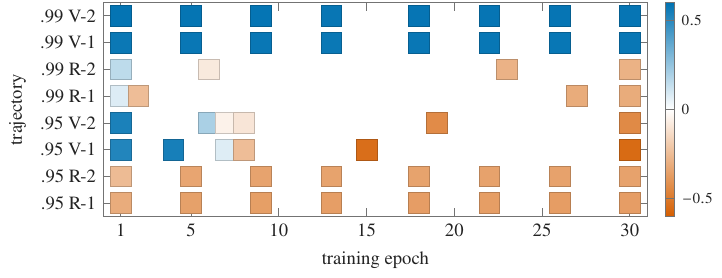}
        \caption{Evidence margin.}
    \end{subfigure}
    \caption{\textbf{Evidence follows shortcut reliance across complete training trajectories.}
    The two panels show the same eight checkpoint sequences through an external behavioral measure and the \decaf evidence margin.}
    \label{fig:app-e-heatmaps}
\end{figure*}

\subsection{Fragility: intervention separation and reveal geometry}
\label{app:f-extensions}

The main text validates $F$ against independently measured off-path prediction change. Figure~\ref{fig:app-f-extensions} provides three complementary checks. First, $F$ cleanly orders robust, neutral, and fragile interventions. Second, similar total response can correspond to different endpoint-relative roles across architectures. Third, the robust--neutral--fragile ordering transfers from the primary covariance-matched reveal to diagonal covariance, trace-matched pixel Gaussian, and three covariance-power geometries. The controlled fragility result is therefore not tied to one second-order reveal geometry, although the numerical summaries remain protocol-relative.

\begin{figure*}[t]
    \centering
    \begin{subfigure}[t]{0.27\textwidth}
        \centering
        \includegraphics[
            height=3.5cm,
            width=\linewidth,
            keepaspectratio
        ]{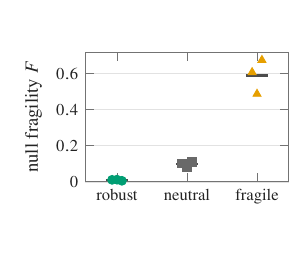}
        \caption{Intervention separation.}
    \end{subfigure}
    \hfill
    \begin{subfigure}[t]{0.27\textwidth}
        \centering
        \includegraphics[
            height=3.5cm,
            width=\linewidth,
            keepaspectratio
        ]{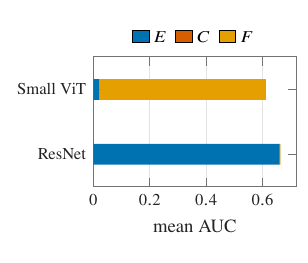}
        \caption{Similar magnitude, different roles.}
    \end{subfigure}
    \hfill
    \begin{subfigure}[t]{0.38\textwidth}
        \centering
        \includegraphics[
            height=3.0cm,
            width=\linewidth,
            keepaspectratio
        ]{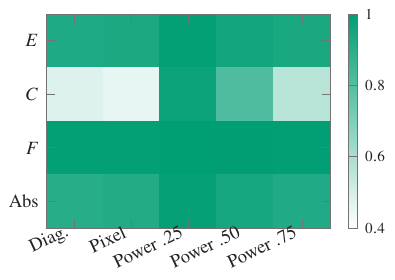}
        \caption{Cross-geometry rank transfer.}
    \end{subfigure}
    \caption{\textbf{Additional fragility checks.}
    (a) $F$ separates robust, neutral, and fragile training.
    (b) Similar ordinary response can be endpoint evidence in one model and endpoint-null fragility in another.
    (c) The controlled ordering transfers across held-out reveal geometries.}
    \label{fig:app-f-extensions}
\end{figure*}
\begin{figure*}[!t]
    \centering

    \begin{subfigure}[t]{0.285\textwidth}
        \centering
        \includegraphics[
            height=3cm,
            width=\linewidth,
            keepaspectratio
        ]{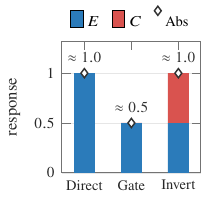}
        \caption{Attenuation vs.\ inversion.}
    \end{subfigure}
    \hfill
    \begin{subfigure}[t]{0.285\textwidth}
        \centering
        \includegraphics[
            height=3cm,
            width=\linewidth,
            keepaspectratio
        ]{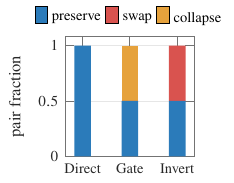}
        \caption{Independent label behavior.}
    \end{subfigure}
    \hfill
    \begin{subfigure}[t]{0.285\textwidth}
        \centering
        \includegraphics[
            height=2.7cm,
            width=\linewidth,
            keepaspectratio
        ]{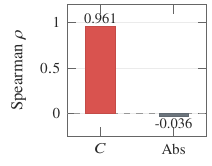}
        \caption{$C$ succeeds where $\Abs$ fails.}
    \end{subfigure}

    \vspace{5pt}

    \begin{subfigure}[t]{0.285\textwidth}
        \centering
        \includegraphics[
            height=3cm,
            width=\linewidth,
            keepaspectratio
        ]{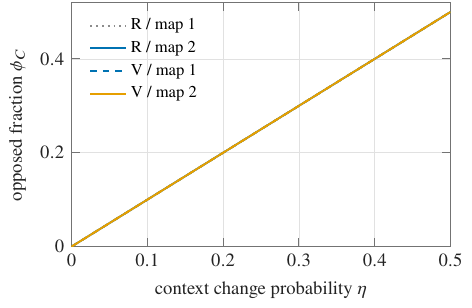}
        \caption{Opposed-fraction calibration.}
    \end{subfigure}
    \hfill
    \begin{subfigure}[t]{0.285\textwidth}
        \centering
        \includegraphics[
            height=3cm,
            width=\linewidth,
            keepaspectratio
        ]{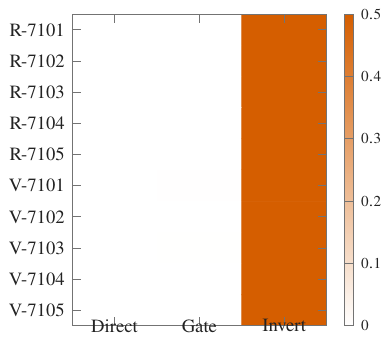}
        \caption{Architecture and seed consistency.}
    \end{subfigure}
    \hfill
    \begin{subfigure}[t]{0.285\textwidth}
        \centering
        \includegraphics[
            height=3cm,
            width=\linewidth,
            keepaspectratio
        ]{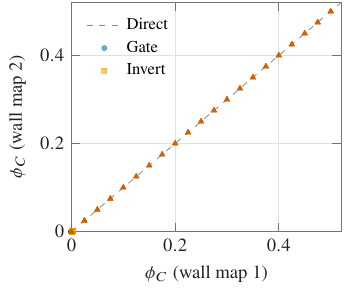}
        \caption{Counterfactual-map transfer.}
    \end{subfigure}

    \caption{\textbf{Additional contradiction checks.}
    Top: contradiction separates attenuation from inversion, agrees with independently measured label behavior, and succeeds where ordinary magnitude does not.
    Bottom: the opposed fraction is calibrated and stable across architectures, seeds, and alternative wall-color counterfactual maps.}
    \label{fig:app-c-extensions}
\end{figure*}

\subsection{Contradiction: regime separation and calibration}
\label{app:c-extensions}

The main text shows that $C$ tracks independently measured label reversal. Figure~\ref{fig:app-c-extensions} expands this result. Direct preserves the effect, Gate suppresses it, and Invert reverses it; only inversion creates substantial opposed mass. The induced label behavior independently separates into preservation, collapse, and swap. Across architectures, seeds, and wall-color maps, the opposed fraction remains well calibrated. Raw sign changes can be frequent when the stage effect is numerically close to zero, but the corresponding opposed mass remains small, separating magnitude-aware contradiction from sign counting.

\FloatBarrier
\section{ImageNet-9 Experimental Details}
\label{app:imagenet9-setup}

This appendix gives the construction and secondary diagnostics behind Section~\ref{sec:imagenet9}. The main text keeps only the tests needed for the central argument. Here we specify the data splits, model zoo, paired background interventions, independent behavioral indicators, magnitude-controlled evaluation, spatial-attribution baselines, and cross-architecture comparison.

\subsection{Data variants and disjoint splits}

We use the official ImageNet-9 Backgrounds Challenge variants~\citep{xiao2021noise}: Original, Mixed-Same, Mixed-Rand, Mixed-Next, Only-FG, and the challenge backgrounds. The variants preserve foreground identity while changing how the background is constructed. We match variants by foreground identifier and retain 4,050 foregrounds for which the required variants are available.

A deterministic split keeps the roles of the data separate. We reserve 1,644 paired foregrounds for the broad model-level response scan and an 820-foreground deep pool for the sample-level benchmark; 768 foregrounds from this pool are used for the expensive baseline comparison. The remaining foregrounds are not used in that sample-level benchmark. No foreground appears in more than one split.

\subsection{Model zoo and common score space}

The 72-model zoo contains two complementary groups. The first consists of 24 off-the-shelf ImageNet-1k classifiers from torchvision and timm. These models were trained independently of our audit and span ResNet, ResNeXt, WideResNet, DenseNet, EfficientNet, RegNet, ConvNeXt, DeiT, ViT, BEiT, Swin, MaxViT, and CoAtNet architectures.

The second group contains 48 models trained directly on ImageNet-9. We fine-tune six backbones---ResNet-50, ConvNeXt-Tiny, EfficientNet-B3, RegNetY-8GF, Swin-T, and ViT-B/16---on four versions of the training data: Original, Mixed-Rand, Mixed-Next, and Only-FG, using two random seeds. This $6\times4\times2$ design broadens the range of learned background dependence while retaining the same ImageNet-9 prediction task.

We train the fine-tuned models for eight epochs with AdamW, cosine decay, one warm-up epoch, BF16 computation, random resized crops, and horizontal flips. ResNet-50, EfficientNet-B3, and RegNetY-8GF use batch size 256 and learning rate $3\times10^{-4}$. ConvNeXt-Tiny, Swin-T, and ViT-B/16 use batch size 128 and learning rate $10^{-4}$.

All models are evaluated in the same nine-class score space. For an ImageNet-1k model, we compute its 1,000-way softmax and sum the probabilities belonging to each official ImageNet-9 superclass. We do not sum logits or apply a second softmax. The ImageNet-9 fine-tuned models output nine logits directly. In both groups, the scalar analyzed by \decaf is the probability assigned to the true ImageNet-9 superclass.

\subsection{Paired background interventions and reveal paths}

For each foreground, we construct two paired background interventions. Same--Rand compares Mixed-Same with Mixed-Rand: the foreground is unchanged, while a same-class background is replaced by a random-class background. This pair is used for the evidence and fragility analyses. Same--Next compares Mixed-Same with Mixed-Next, where the replacement background comes from superclass $(y+1)\bmod 9$. This pair creates a directional background change for the contradiction analysis.

The primary reveal begins from a shared blurred midpoint. For endpoints $\mathbf x^{+}$ and $\mathbf x^{-}$,
\begin{equation}
    \mathbf x^{0}
    =
    \operatorname{Blur}\!\left(
        \frac{\mathbf x^{+}+\mathbf x^{-}}{2}
    \right),
\end{equation}
and each branch is linearly revealed from $\mathbf x^{0}$ to its endpoint. We evaluate nine stages,
$t\in\{0,0.125,\ldots,1\}$. The main endpoint threshold is $\varepsilon=0.02$ in true-class probability.

To test whether conclusions depend on the reveal path, we also use nested-patch reveal on an $8\times8$ grid. Patches are ordered by endpoint-difference energy, with a fixed random tie-break. Two independently tie-broken orders are evaluated. Both branches use the same patch order and the same neutral starting image. Appendix~\ref{app:imagenet9-results} reports the corresponding path-order and threshold checks.

\subsection{Behavioral indicators defined without \decaf}
\label{app:imagenet9-labels}

The sample-level benchmark defines three observable behavioral indicators directly from model predictions. None uses $E$, $C$, or $F$, and the indicators are allowed to overlap.

\paragraph{Evidence behavior.}
Set $Y_E=1$ when Mixed-Same is classified correctly and replacing its background with Mixed-Rand either changes the predicted class or lowers the true-class probability by at least $0.20$. This records a consequential change caused by replacing the background.

\paragraph{Contradiction behavior.}
Set $Y_C=1$ when changing only the background makes the model switch from the foreground class to the class associated with the new Mixed-Next background. This is an independently observed directional reversal in the model's prediction.

\paragraph{Fragility behavior.}
We first restrict attention to Same--Rand pairs whose fully revealed endpoints produce little difference, $|d|<0.02$. The endpoint pair alone cannot tell us whether such a case is otherwise sensitive, so we test it separately with additional background changes that are not part of the \decaf reveal. These changes use Gaussian blur, Gaussian noise, pixelation, color shift, and patch shuffle, each at two severities. Set $Y_F=1$ if any of them changes the predicted class or shifts the true-class probability by at least $0.20$.

These indicators are external behavioral checks, not definitions of the \decaf components. Their purpose is to ask whether the response roles obtained from the paired trajectory agree with behavior measured independently of that trajectory decomposition.

\subsection{Magnitude-controlled evaluation}

The deep benchmark uses 32 models and 768 foregrounds, giving 24,576 model--image units. We use two complementary ways to control ordinary response magnitude.

First, we divide the units into 20 quantile bins of $\Abs$. Within a bin, a behavioral AUROC is computed only when both positive and negative examples are present for that indicator. The main text reports both the valid-bin summary and the stricter common-support comparison in which all three indicators are represented.

Second, we construct one-to-one matched comparisons between cases with different behavioral indicator patterns but nearly identical ordinary response magnitudes. We require the relative difference in $\Abs$ to be at most $5\%$, yielding 8,289 matched comparisons.

Because $Y_E$, $Y_C$, and $Y_F$ may overlap, the matched evaluation does not force every case into one ground-truth class. For a case with at least one active indicator, we examine the largest score among the three response-role coordinates. The case receives full credit when the largest coordinate corresponds to an active behavioral indicator. If several coordinates tie for the maximum, credit is divided across the tied coordinates. The reported matched-pair accuracy averages this role-agreement score over the matched cases. Ordinary magnitude has only one scalar value and therefore cannot prefer evidence, contradiction, or fragility.

The implementation also records a separate pairwise-ranking diagnostic: when an independent behavioral indicator differs across a matched pair, it asks whether the corresponding response-role score changes in the same direction. This ranking diagnostic is distinct from the role-agreement accuracy reported in Section~6.2.

\subsection{Spatial-attribution baselines}

The deep benchmark includes six standard spatial-attribution methods: Input$\times$Gradient, Integrated Gradients with 16 steps, SmoothGrad with 16 noise samples, BlurIG with 12 blur levels, a $7\times7$ Occlusion grid, and RISE with 256 masks. For each method, we sum absolute attribution inside the foreground--background difference region. These scores describe where the prediction is sensitive; they are not used to define the three behavioral indicators above.

\subsection{Cross-architecture comparison and spatial complementarity}

\begin{figure}[!t]
    \centering
    \includegraphics[width=0.5\linewidth]{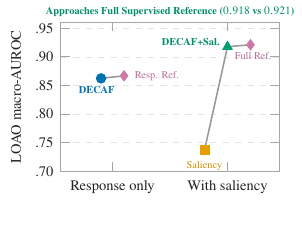}
    \caption{\textbf{Cross-architecture transfer and spatial complementarity.}
    All learned results hold out one architecture family during fitting.
    Response features from \decaf nearly match the supervised response-statistics reference, while adding spatial attribution provides a further gain and approaches the full supervised reference.}
    \label{fig:app-in9-cross-family}
\end{figure}

For the learned comparison, every representation is evaluated with the same supervised protocol: train on all but one architecture family and evaluate on the held-out family. The \decaf representation uses $(M,E,C,F)$. A response-statistics reference uses $(M,\Abs,\mathrm{Net},\mathrm{SignFlip},\mathrm{active\ rate})$. The spatial representation uses the six attribution scores above. We also evaluate \decaf combined with the spatial scores and a full reference combining response statistics with spatial attribution.

Figure~\ref{fig:app-in9-cross-family} contains the panel moved from the main text. Using response information alone, \decaf reaches leave-one-architecture-family-out macro-AUROC $0.862$, close to the supervised response-statistics reference at $0.867$. Spatial attribution alone reaches $0.737$. Combining \decaf with spatial attribution reaches $0.918$, close to the full supervised reference at $0.921$. This is best read as a complementarity check: the response decomposition and spatial attribution retain different information.

\subsection{Statistical reporting}

For direct behavioral comparisons, we report AUROC and AUPRC. For learned comparisons, we report leave-one-architecture-family-out results so that the evaluated architecture family is not used to fit the probe. Undefined AUROCs remain missing rather than being replaced by chance. Where bootstrap intervals are reported, models---rather than individual model--image rows---are the resampling unit.

\FloatBarrier
\section{ImageNet-9 Protocol Robustness Checks}
\label{app:imagenet9-results}

\begin{figure*}[!t]
\centering
\begin{subfigure}[t]{0.45\textwidth}
    \centering
    \includegraphics[width=\linewidth]{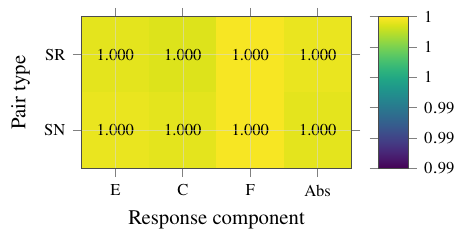}
    \caption{Patch-order stability.}
\end{subfigure}
\hspace{0.06\textwidth}
\begin{subfigure}[t]{0.43\textwidth}
    \centering
    \includegraphics[width=\linewidth]{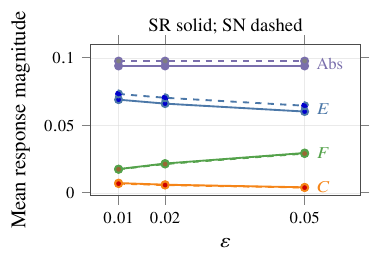}
    \caption{Endpoint-threshold sensitivity.}
\end{subfigure}
\caption{\textbf{Robustness of the ImageNet-9 protocol audit.}
(a) The two independently tie-broken nested-patch orders agree almost perfectly for every response component.
(b) Varying the endpoint threshold changes the allocation between active and endpoint-null response, but not the qualitative protocol conclusion.}
\label{fig:app-in9-sensitivity}
\end{figure*}

The protocol audit in Section~\ref{sec:imagenet9} changes the reveal path while keeping the factual--counterfactual endpoints fixed. We retain two checks that test whether the reported effect is caused by an arbitrary implementation choice.

First, the nested-patch reveal is evaluated with two independently tie-broken patch orders. Their response summaries agree almost perfectly, showing that the blend-to-patch contrast is not an accident of one patch sequence. Second, we recompute the decomposition at $\varepsilon\in\{0.01,0.02,0.05\}$. Changing the threshold reallocates response between the active and endpoint-null branches as expected, but preserves the qualitative conclusion that patch reveal increases total response without increasing evidence.

\FloatBarrier
\section{Forward-Only Attribution: Complete Results and Boundaries}
\label{app:forward-attribution}

This appendix expands the compact comparison in
Section~\ref{sec:forward-attribution}.
It records the strict common-support protocol, complete baseline results,
an official-protocol FunnyBirds reproduction check, the endpoint-versus-trajectory
ablation, full-scale ImageNet validation, measured compute, large-model scaling,
and the PartImageNet boundary case.

\subsection{Protocol and strict common support}

FunnyBirds and ImageNet-1k use the same three architectures: ResNet-50, VGG-16, and ViT-B/16. FunnyBirds evaluates semantic-part rankings against two held-out operators, Telea inpainting and background-texture replacement. ImageNet-1k IDSDS partitions every image into a fixed $4\times4$ grid and evaluates the ranking of 16 patch scores against in-domain single-deletion effects. We first filter to images classified correctly by each model. We then intersect image IDs across the methods included in the strict comparison. The resulting support contains 499, 497, and 488 FunnyBirds images, and 7,663, 7,189, and 8,285 ImageNet images, respectively.

For each image, the primary metric is Spearman correlation between feature attribution and the benchmark target. We average within each model and then macro-average equally across architectures. All intervals use 1,000 paired image-cluster bootstrap replicates. We never average the FunnyBirds and ImageNet columns into one leaderboard.

Endpoint $M_k=|d_k|$ is an endpoint-only reference. On ImageNet IDSDS, the endpoint audit gives $\max_k|d_k-g_k|=0$ and confirms that every stored endpoint effect uses the same deletion pair as the evaluation target. The direct target-derived $|g_k|$ and independently persisted DECAF $M_k$ differ only by a mean numerical discrepancy of $6.0\times10^{-6}$.

\subsection{Complete cross-dataset method comparison}
\label{app:forward-attribution-comparison}

\begin{table*}[t]
\centering
\caption{Complete strict-common-support attribution results. Entries are equal-architecture macro-average Spearman correlations with $95\%$ paired image-cluster bootstrap intervals. Endpoint $M$ is an endpoint-only reference. KernelSHAP-512 is a deletion-game reference.}
\label{tab:forward-attribution-complete}
\footnotesize
\setlength{\tabcolsep}{4.5pt}
\begin{tabular}{llcc}
\toprule
Method & Access & FunnyBirds $\rho$ [95\% CI] & ImageNet-1k IDSDS $\rho$ [95\% CI] \\
\midrule
\multicolumn{4}{l}{\emph{DECAF}} \\
\decaf-3 & Forward only & 0.372 [0.353, 0.392] & 0.359 [0.355, 0.364] \\
\decaf-5 & Forward only & 0.403 [0.385, 0.422] & 0.367 [0.362, 0.371] \\
\decaf-9 & Forward only & 0.406 [0.388, 0.425] & 0.379 [0.375, 0.383] \\
\addlinespace
\multicolumn{4}{l}{\emph{General-purpose attribution}} \\
Input$\times$Gradient & Backward & 0.019 [$-$0.002, 0.042] & 0.098 [0.094, 0.101] \\
IG-16 & Backward & 0.266 [0.246, 0.285] & 0.238 [0.235, 0.242] \\
IG-32 & Backward & 0.271 [0.251, 0.290] & 0.242 [0.238, 0.245] \\
IG-U-32 & Backward & 0.200 [0.178, 0.222] & 0.295 [0.292, 0.299] \\
DeepLIFT & Backward & 0.197 [0.176, 0.217] & 0.341 [0.338, 0.344] \\
GradientSHAP & Backward & 0.226 [0.204, 0.246] & 0.236 [0.233, 0.240] \\
SmoothGrad-16 & Backward & $-$0.045 [$-$0.065, $-$0.023] & $-$0.015 [$-$0.018, $-$0.011] \\
RISE-512 & Sampling & 0.302 [0.283, 0.321] & 0.179 [0.175, 0.183] \\
RISE-U-512 & Sampling & 0.294 [0.274, 0.315] & 0.111 [0.107, 0.115] \\
\addlinespace
\multicolumn{4}{l}{\emph{Endpoint-only reference}} \\
Endpoint $M$ & Endpoint only & 0.324 [0.303, 0.344] & 0.371 [0.366, 0.377] \\
\addlinespace
\multicolumn{4}{l}{\emph{Deletion-game reference}} \\
KernelSHAP-512 & Sampling & 0.299 [0.278, 0.319] & 0.447 [0.443, 0.451] \\
\bottomrule
\end{tabular}
\end{table*}

\paragraph{Published-benchmark sanity check.}
On the original full-scale IDSDS ResNet-50 setting, our IG and IG-U scores are 0.194 and 0.252, closely matching the published values of 0.196 and 0.255~\citep{hesse2024benchmarking}. This agreement provides an external check on our ImageNet implementation. Our FunnyBirds evaluation instead uses held-out intervention operators and strict common support, so it is not intended to reproduce the dataset's native part-based evaluation protocol~\citep{hesse2023funnybirds}. We therefore compare its baseline values only within our registered protocol.

The stronger IG baseline changes across datasets. IG-32 exceeds IG-U-32 on FunnyBirds, while IG-U-32 is stronger on IDSDS. \decaf-5 exceeds both variants on both datasets. KernelSHAP exhibits the opposite transfer pattern: it is strongest on IDSDS, whose target is the deletion game it queries, but is weaker than endpoint $M$ and DECAF under the held-out FunnyBirds operators.

\paragraph{Native-protocol reproduction check.}
The FunnyBirds column above evaluates attribution transfer to two held-out
intervention operators, rather than the dataset's native Single Deletion
protocol. We therefore ran a separate implementation check using the official
FunnyBirds setting: the released checkpoints, native semantic part removals,
and RISE with 6,000 masks, an $8\times8$ coarse grid, and mask probability
$p=0.1$~\citep{hesse2023funnybirds}. As shown in
Table~\ref{tab:funnybirds-native-reproduction}, the reproduced RISE scores
differ from the published values by less than $0.005$ for every architecture.
This check supports the implementation fidelity of the RISE baseline used in
our benchmark; it does not imply that the held-out score of $0.302$ should
numerically match the native Single Deletion scores, because the two
evaluations use different intervention targets and RISE configurations.

\begin{table}[t]
\centering
\caption{
FunnyBirds native Single Deletion reproduction and native-target audit.
\textbf{Top:} the official RISE-6000 setting reproduces the published
Single Deletion scores within $0.005$ on all three architectures.
\textbf{Bottom:} raw Spearman correlation with the same native part-removal
target. Endpoint $M$ is included to show the strong target alignment of this
evaluation. These native-target results are a reproduction and boundary check,
not a replacement for the held-out FunnyBirds comparison in
Table~\ref{tab:forward-attribution-main}.
}
\label{tab:funnybirds-native-reproduction}
\small
\setlength{\tabcolsep}{4.5pt}

\begin{tabular}{lccc}
\toprule
Architecture & Published RISE SD & Reproduced RISE SD & $|\Delta|$ \\
\midrule
ResNet-50 & 0.560 & 0.558 & 0.002 \\
VGG16     & 0.730 & 0.726 & 0.004 \\
ViT-B/16  & 0.790 & 0.788 & 0.002 \\
\midrule
\multicolumn{4}{c}{\textit{Native-target raw Spearman $\rho$}} \\
\midrule
Method & ResNet-50 & VGG16 & ViT-B/16 \\
\midrule
Endpoint $M$ & 0.757 & 1.000 & 0.983 \\
\decaf-3     & 0.484 & 0.926 & 0.909 \\
\decaf-5     & 0.565 & 0.898 & 0.864 \\
\decaf-9     & 0.616 & 0.901 & 0.850 \\
RISE-6000    & 0.115 & 0.452 & 0.576 \\
\bottomrule
\end{tabular}
\end{table}

The native and held-out FunnyBirds experiments answer different questions.
The native target is defined by the same semantic part-removal contrast used
to construct the endpoint, whereas the main benchmark evaluates transfer to
held-out inpainting and texture-replacement interventions. The native audit
therefore verifies baseline implementation and exposes the target-aligned
boundary; the held-out benchmark remains the test of attribution transfer.

\subsection{Endpoint versus trajectory}
\label{app:endpoint-versus-trajectory}

\begin{table}[t]
\centering
\caption{Paired endpoint-versus-trajectory tests. Differences are DECAF minus endpoint $M$. Positive values indicate ordinary-attribution information beyond the endpoint.}
\label{tab:endpoint-trajectory-paired}
\footnotesize
\setlength{\tabcolsep}{4pt}
\begin{tabular}{llrrr}
\toprule
Dataset & Contrast & Mean $\Delta$ & 95\% CI & Model wins \\
\midrule
FunnyBirds & \decaf-3 $-M$ & +0.0488 & [0.0334, 0.0642] & 2/3 \\
FunnyBirds & \decaf-5 $-M$ & +0.0797 & [0.0637, 0.0962] & 2/3 \\
FunnyBirds & \decaf-9 $-M$ & +0.0828 & [0.0671, 0.0993] & 2/3 \\
\addlinespace
ImageNet-1k & \decaf-3 $-M$ & $-$0.0119 & [$-$0.0142, $-$0.0096] & 1/3 \\
ImageNet-1k & \decaf-5 $-M$ & $-$0.0045 & [$-$0.0072, $-$0.0017] & 1/3 \\
ImageNet-1k & \decaf-9 $-M$ & +0.0073 & [0.0045, 0.0101] & 2/3 \\
\bottomrule
\end{tabular}
\end{table}

FunnyBirds is the clean trajectory-value test because its evaluation operators differ from the explanation endpoint. All three DECAF grids significantly exceed $M$. Five stages capture $96.3\%$ of the nine-stage gain over the endpoint. IDSDS evaluates the same deletion contrast that defines $d_k$. Endpoint $M$ is therefore already highly aligned with its target. Three and five stages preserve most of that signal, while nine stages add a small but stable gain.

\begin{table*}[t]
\centering
\caption{Endpoint-versus-trajectory ablation by architecture. Best $\Delta$ is the largest DECAF-3/5/9 Spearman minus endpoint $M$.}
\label{tab:endpoint-trajectory-per-model}
\footnotesize
\setlength{\tabcolsep}{4.2pt}
\begin{tabular}{llrrrrr}
\toprule
Dataset & Model & Endpoint $M$ & \decaf-3 & \decaf-5 & \decaf-9 & Best $\Delta$ \\
\midrule
FunnyBirds & ResNet-50 & 0.388 & 0.425 & 0.447 & 0.451 & +0.064 \\
FunnyBirds & VGG-16 & 0.306 & 0.283 & 0.298 & 0.293 & $-$0.008 \\
FunnyBirds & ViT-B/16 & 0.277 & 0.409 & 0.465 & 0.475 & +0.197 \\
ImageNet-1k & ResNet-50 & 0.384 & 0.362 & 0.370 & 0.387 & +0.002 \\
ImageNet-1k & VGG-16 & 0.671 & 0.640 & 0.640 & 0.644 & $-$0.027 \\
ImageNet-1k & ViT-B/16 & 0.058 & 0.077 & 0.090 & 0.105 & +0.047 \\
\bottomrule
\end{tabular}
\end{table*}

The trajectory gain is architecture-dependent in the tested models. It is largest for ViT-B/16 on both datasets, modest for ResNet-50, and negative for VGG-16. This pattern is empirical rather than a general architecture law, but it shows that the value of intermediate responses is not uniform across model classes.

The separate native FunnyBirds audit in
Appendix~\ref{app:forward-attribution-comparison} provides the complementary
endpoint-aligned case: when the evaluation target is the exact native
part-removal contrast, endpoint $M$ reaches $\rho=0.913$ and exceeds all
trajectory summaries.

Together, the two FunnyBirds evaluations isolate the role of the trajectory:
the endpoint best recovers its own intervention effect, while intermediate
responses add value when attribution must transfer to a different intervention.

\subsection{Full-scale ImageNet validation}

\begin{table}[t]
\centering
\caption{Full 50,000-image ImageNet validation scale check. Values are equal-architecture macro-average IDSDS over 115,876 correctly classified model--image units.}
\label{tab:idsds-full50k}
\footnotesize
\setlength{\tabcolsep}{4pt}
\begin{tabular}{lrrrr}
\toprule
Method & Macro $\rho$ [95\% CI] & ResNet-50 & VGG-16 & ViT-B/16 \\
\midrule
Endpoint $M$ & 0.3708 [0.3685, 0.3731] & 0.3778 & 0.6731 & 0.0615 \\
\decaf-5 & 0.3633 [0.3613, 0.3653] & 0.3609 & 0.6395 & 0.0894 \\
IG-U-32 & 0.2942 [0.2926, 0.2958] & 0.2522 & 0.4939 & 0.1366 \\
IG-32 & 0.2397 [0.2380, 0.2411] & 0.1941 & 0.4002 & 0.1247 \\
\bottomrule
\end{tabular}
\end{table}

The 10,000-image conclusions are not a favorable-subset artifact. The full validation set preserves the ordering $M > \decaf\text{-}5 > \mathrm{IG\text{-}U\text{-}32} > \mathrm{IG\text{-}32}$. The full-scale paired difference between \decaf-5 and $M$ is $-0.0075$ with a $95\%$ interval of $[-0.0086,-0.0062]$.

\subsection{Measured compute and large-model scaling}

\begin{table*}[!tb]
\centering
\caption{Measured ImageNet-1k compute. Rows per image count model-forward input rows after batching. Timing is compute-only and macro-averaged across ResNet-50, VGG-16, and ViT-B/16.}
\label{tab:forward-attribution-compute}
\footnotesize
\setlength{\tabcolsep}{4.5pt}
\begin{tabular}{llrrrr}
\toprule
Method & Access & Rows/image & Backward? & ms/image $\downarrow$ & Peak GiB $\downarrow$ \\
\midrule
\decaf-3 & Forward only & 51 & No & 22.1 & 11.2 \\
\decaf-5 & Forward only & 85 & No & 36.7 & 11.2 \\
\decaf-9 & Forward only & 153 & No & 65.9 & 11.2 \\
Input$\times$Gradient & Backward & 1 & Yes & 1.4 & 2.3 \\
DeepLIFT & Backward & 2 & Yes & 3.1 & 6.6 \\
GradientSHAP & Backward & 16 & Yes & 14.4 & 25.9 \\
IG-32 & Backward & 32 & Yes & 28.7 & 50.4 \\
IG-U-32 & Backward & 32 & Yes & 28.7 & 50.4 \\
SmoothGrad-16 & Backward & 16 & Yes & 16.1 & 25.9 \\
RISE-512 & Sampling & 512 & No & 216.6 & 14.4 \\
KernelSHAP-512 & Sampling & 512 & No & 216.4 & 14.4 \\
\bottomrule
\end{tabular}
\end{table*}

\decaf-3 is the speed-oriented Pareto point. Relative to IG-32, it has higher quality on both datasets, is $1.30\times$ faster, and uses $4.50\times$ less peak memory. \decaf-5 trades additional latency for higher quality while keeping the same inference-scale memory. KernelSHAP-512 is $5.89\times$ slower than \decaf-5 and uses six times as many forward rows.

\subsection{PartImageNet as a task-aligned boundary case}

PartImageNet supplies semantic part masks and evaluates held-out part-removal effects. Direct part removal and coalition methods are therefore unusually aligned with the evaluation target. We retain the benchmark as a boundary case rather than using it to define the main general-purpose comparison.

\begin{table}[!tb]
\centering
\caption{PartImageNet strict-common-support Spearman. Part-removal and coalition methods exploit the supplied semantic part groups and are directly aligned with the held-out part-removal target.}
\label{tab:partimagenet-boundary}
\footnotesize
\setlength{\tabcolsep}{4pt}
\begin{tabular}{lr}
\toprule
Method & Spearman [95\% CI] \\
\midrule
Part-LIME-1000 & 0.478 [0.458, 0.497] \\
Part Occlusion & 0.477 [0.458, 0.496] \\
Exact Part-Shapley & 0.433 [0.412, 0.454] \\
KernelSHAP-512 & 0.426 [0.405, 0.447] \\
Endpoint $M$ & 0.364 [0.342, 0.385] \\
\decaf-9 & 0.363 [0.341, 0.384] \\
\decaf-5 & 0.358 [0.337, 0.379] \\
\decaf-3 & 0.350 [0.329, 0.370] \\
RISE-512 & 0.299 [0.276, 0.320] \\
IG-32 & 0.289 [0.270, 0.311] \\
\bottomrule
\end{tabular}
\end{table}

This result marks the method boundary clearly. When perfect semantic parts are already supplied and the target is itself a part-removal effect, direct part interventions can be stronger. Even in that setting, the DECAF trajectories remain above standard gradient-path attribution and preserve their forward-only memory advantage.

\FloatBarrier

\section{Covertype Experimental Details and Complete Results}
\label{app:covertype-transfer}

This appendix reports the complete Covertype audit behind Section~\ref{sec:controlled}. The experiment uses a natural tabular base and controlled context--factor channels. Its purpose is not to assume that every treatment succeeds. It measures which response mechanism each trained model actually realizes.

\subsection{Setup and Operational Outcomes}
\label{app:covertype-setup}

We balance classes 1 and 2 from Covertype and retain 240,000 examples. The 54 natural features are augmented by one binary context and one binary candidate factor. The direction module uses context $G$ and factor $Z$; the fragility module uses context $H$ and factor $U$. Natural features and data splits are shared across all mechanism variants.

\begin{table}[t]
\centering
\caption{Covertype benchmark design. Each model family uses three seeds. Direction experiments additionally use shortcut strengths $p\in\{0.75,0.95\}$.}
\label{tab:covertype-setup}
\small
\begin{tabular}{llllr}
\toprule
Module & Regimes & Held-out operational behavior & Families & Models \\
\midrule
Direction & Direct, Gate, Invert & preserve / collapse / invert & 5 & 90 \\
Fragility & Robust, Mild, Fragile & endpoint-null prediction change & 5 & 45 \\
\midrule
Total & & & & 135 \\
\bottomrule
\end{tabular}
\end{table}

For the direction module, the endpoint is $G=+1$. We query the factor effect under $G=+1$ and $G=-1$, then define preservation, collapse, and inversion from the held-out signed responses. For the fragility module, the endpoint is $H=+1$. The primary behavioral target used in Section~\ref{sec:controlled} is the prediction-change rate under $H=-1$ among examples satisfying the endpoint-null gate. This target was stored by the formal experiment and matches the definition of $F$.

\subsection{Complete Behavior Alignment}
\label{app:covertype-complete-ranks}

Table~\ref{tab:covertype-complete-ranks} reports the complete rank comparison. Evidence uses preservation, contradiction uses actual inversion, and fragility uses endpoint-null alternate-context prediction change. The first two columns use 90 direction models; the last uses 45 fragility models. Parenthetical sample counts mark methods with partial coverage.

\begin{table*}[t]
\centering
\caption{Spearman correlation with realized behavior. SHAP interaction is available for random forests and XGBoost; KernelSHAP and LIME use their fixed formal subsets. $F$ values use the endpoint-null behavior target.}
\label{tab:covertype-complete-ranks}
\footnotesize
\setlength{\tabcolsep}{5pt}
\begin{tabular}{lrrr}
\toprule
Method & Preservation $\rho$ & Actual inversion $\rho$ & Endpoint-null change $\rho$ \\
\midrule
DECAF component & \textbf{0.864} & \textbf{0.987} & \textbf{0.974} \\
Endpoint $M$ & 0.657 & $-0.001$ & 0.148 \\
Abs & 0.804 & $-0.207$ & 0.588 \\
Signed net & \textbf{0.878} & $-0.353$ & 0.043 \\
SignFlip & $-0.919$ (78) & 0.888 (78) & 0.842 (39) \\
OppMass & $-0.412$ & 0.974 & 0.965 \\
Native SHAP & 0.781 & $-0.205$ & 0.481 \\
SHAP interaction & $-0.251$ (54) & 0.093 (54) & 0.069 (27) \\
KernelSHAP & 0.811 (30) & $-0.219$ (30) & 0.441 (15) \\
Global PFI & 0.664 & $-0.029$ & 0.676 \\
Context-conditioned PFI & 0.298 & 0.311 & 0.739 \\
PDP/ALE interaction & $-0.100$ & 0.827 & 0.937 \\
LIME & 0.845 (30) & $-0.384$ (30) & 0.506 (15) \\
\bottomrule
\end{tabular}
\end{table*}

For the two headline semantic coordinates, joint family/seed cluster bootstrap gives
\[
\rho(C,V_{\mathrm{inv}})=0.9868,
\qquad 95\%\ \mathrm{CI}=[0.9725,0.9967],
\]
and
\[
\rho(F,V_{\mathrm{null}})=0.9741,
\qquad 95\%\ \mathrm{CI}=[0.9416,0.9883].
\]
Evidence reaches $\rho(E,V_{\mathrm{keep}})=0.8642$ with 95\% interval $[0.8318,0.8955]$.

\subsection{The Same Treatment Produces Different Realized Mechanisms}
\label{app:covertype-family-audit}

Table~\ref{tab:covertype-family-audit} shows why treatment labels cannot serve as mechanism ground truth. The Invert generator produces genuine reversal in random forests and multilayer perceptrons, but mostly collapse in the remaining families. The Fragile generator becomes endpoint evidence in logistic regression, while nonlinear families retain substantial endpoint-null mass.

\begin{table*}[!tb]
\centering
\caption{Model-family audit. Direction columns average the Invert regime across two strengths and three seeds. Fragility columns average the Fragile regime across three seeds. ``Active'' is the endpoint-active fraction; ``Null change'' is the alternate-context prediction-change rate on endpoint-null examples.}
\label{tab:covertype-family-audit}
\resizebox{\textwidth}{!}{%
\begin{tabular}{lrrr|rrrrr}
\toprule
& \multicolumn{3}{c|}{Invert treatment} & \multicolumn{5}{c}{Fragile treatment} \\
Model family & $C$ & Invert rate & Collapse rate & $M$ & $E$ & $F$ & Active & Null change \\
\midrule
HistGradientBoosting & 0.000 & 0.000 & 1.000 & 0.020 & 0.070 & 0.180 & 0.340 & 0.623 \\
Logistic regression & 0.000 & 0.000 & 1.000 & 0.299 & 0.299 & 0.0001 & 0.989 & 0.000 \\
MLP & 0.088 & 0.743 & 0.257 & 0.028 & 0.084 & 0.080 & 0.350 & 0.219 \\
Random forest & 0.024 & 0.816 & 0.169 & 0.057 & 0.170 & 0.062 & 0.673 & 0.327 \\
XGBoost & 0.000 & 0.002 & 0.998 & 0.020 & 0.091 & 0.163 & 0.346 & 0.552 \\
\bottomrule
\end{tabular}%
}
\end{table*}

Within the Invert treatment alone, $C$ still tracks the amount of realized inversion with $\rho=0.922$. Among families with nonconstant inversion behavior, the correlations are 0.997 for histogram gradient boosting, 0.992 for random forests, 0.971 for XGBoost, and 0.940 for the MLP. Logistic regression has zero inversion throughout, so its within-family correlation is undefined.

\subsection{Fixed Semantic Readout and Magnitude Conditioning}
\label{app:covertype-fixed-semantic}

The fixed semantic benchmark assigns $E$, $C$, and $F$ to Evidence, Contradiction, and Fragility without fitting a meta-classifier. Table~\ref{tab:covertype-fixed-semantic} reports both the unconditional result and the within-Abs-bin result. The latter controls response magnitude through 12 quantile bins.

\begin{table*}[!tb]
\centering
\caption{Mechanism identification with fixed semantics and after conditioning on ordinary magnitude. Learned references use mechanism labels and are not access-matched to fixed DECAF.}
\label{tab:covertype-fixed-semantic}
\footnotesize
\begin{tabular}{lrrrr}
\toprule
Method & \makecell[r]{Fixed\\macro-AUROC} & \makecell[r]{Fixed\\macro-AUPRC} & \makecell[r]{Within-Abs\\macro-AUROC} & \makecell[r]{Within-Abs\\macro-AUPRC} \\
\midrule
\makecell[l]{DECAF / M / Abs /\\ Net / OppMass / SignFlip} & \textbf{0.816} & \textbf{0.768} & \textbf{0.956} & \textbf{0.948} \\
Endpoint $M$ & 0.487 & 0.437 & 0.565 & 0.666 \\
Abs & 0.518 & 0.450 & 0.495 & 0.585 \\
Native SHAP & 0.523 & 0.437 & 0.461 & 0.570 \\
SHAP interaction & 0.540 & 0.440 & -- & -- \\
\makecell[l]{Context-conditioned\\PFI} & 0.556 & 0.545 & 0.468 & 0.634 \\
PDP/ALE interaction & 0.539 & 0.438 & 0.559 & 0.739 \\
\addlinespace
\makecell[l]{Strong tabular reference\\(supervised)} & 0.987 & 0.968 & 0.983 & 0.994 \\
\makecell[l]{DECAF probe\\(supervised)} & 0.902 & 0.859 & 0.742 & 0.779 \\
\makecell[l]{Combined empirical\\reference (supervised)} & 0.985 & 0.965 & 0.983 & 0.994 \\
\bottomrule
\end{tabular}
\end{table*}

The fixed DECAF mechanism accuracy is 0.587. The supervised references use leave-one-model-family-out calibration, mechanism labels, and up to 15 input summaries. They bound cross-family decodability but do not replace the label-free comparison.

\subsection{Measured Cost and SHAP-Interaction Audit}
\label{app:covertype-cost}

Table~\ref{tab:covertype-cost-full} reports the registered cumulative worker time. Relative cost divides worker-seconds per model by the DECAF value. Coverage differs for methods whose formal budget uses a model subset.

\begin{table*}[htbp]
\centering
\caption{Complete measured method cost. SHAP interaction uses 128 stratified examples per tree model, split into four 32-example shards.}
\label{tab:covertype-cost-full}
\footnotesize
\begin{tabular}{lrrrr}
\toprule
Method & Models & Worker-s/model & Relative to DECAF & Predicted rows \\
\midrule
DECAF / $M$ / Abs / Net /  & 135 & 1.42 & $1.0\times$ & 25,920,000 \\
 OppMass / SignFlip &  &  &  &  \\
PDP/ALE interaction & 135 & 2.22 & $1.57\times$ & 38,880,000 \\
LIME & 45 & 8.38 & $5.92\times$ & 11,796,480 \\
PFI / context PFI & 135 & 13.32 & $9.41\times$ & 207,360,000 \\
KernelSHAP & 45 & 45.30 & $31.99\times$ & 377,501,760 \\
Native SHAP & 135 & 387.65 & $273.8\times$ & -- \\
Retraining reference & 135 & 19.70 & $13.91\times$ & -- \\
SHAP interaction & 54 & 15097.63 & $10662\times$ & -- \\
\bottomrule
\end{tabular}
\end{table*}

The formal SHAP-interaction run completed 216 shards over 54 tree models. It consumed at least 811,875 CPU-seconds (225.5 CPU-hours), used up to 32 concurrent shards, and required 17.71 hours of elapsed stage time. Its inversion correlation was only $0.093$ with a 95\% interval spanning zero.

\subsection{Matched-Pair Audit and Its Limits}
\label{app:covertype-matched-limit}

The automatic pair search produced only 11 protocol pairs. They cover 29.3\% of primary units, contain no same-family pair, and have a median relative endpoint-magnitude difference of 0.870 (maximum 0.945), although their Abs difference is small. DECAF reaches 0.727 mechanism accuracy on these pairs, but this set does not support a headline claim that both Abs and $M$ are matched. We therefore use the much better populated within-Abs-bin analysis in Table~\ref{tab:covertype-fixed-semantic} and retain the 11-pair result only as an audit.

\FloatBarrier

\bibliography{references}
\bibliographystyle{iclr2027_conference}

\end{document}